\documentclass{article}
\usepackage{fullpage}

\usepackage[title]{appendix}
\usepackage{mathtools}

\DeclarePairedDelimiter\floor{\lfloor}{\rfloor}
\usepackage{tabu}
\usepackage{bm}
\usepackage{hhline}
\usepackage[utf8]{inputenc} 
\usepackage[T1]{fontenc}    
\usepackage{url}            
\usepackage{booktabs}       
\usepackage{amsfonts} 
\usepackage{amsthm}
\usepackage{nicefrac}  
\usepackage{graphicx}
\usepackage{subcaption}
\usepackage{color}
\usepackage{amsfonts,amssymb}
\usepackage{mathtools}
\usepackage{amsfonts}
\usepackage{url}
\usepackage{lipsum}
\usepackage{mysymbol}
\usepackage[ruled]{algorithm2e}
\usepackage{enumitem}
\usepackage{mathtools}  
\mathtoolsset{showonlyrefs}

\newtheorem{ex}{Example}

\newtheorem{theorem}{Theorem}
\newtheorem{remark}{Remark}
\newtheorem{lemma}{Lemma}

\newtheorem{prop}{Proposition}
\newtheorem{assumption}{Assumption}

\newcommand{\gr}{\nabla}
\newcommand{\grw}{\nabla_{\bbw}} 
\newcommand{\grv}{\nabla_{\bbv}}

\newcommand{\sgr}{\widetilde{\nabla}}
 
\newcommand{\sgrw}{\sgr_{\bbw}} 
\newcommand{\sgrv}{\sgr_{\bbv}}
\newcommand{\bpsi}{\bm{\psi}}

\newcommand{\E}{\mathbb{E}}

\newcommand{\vbar}{\overline{\bbv}}
\newcommand{\wbar}{\overline{\bbw}}
\newcommand{\w}{\mathbf{w}}

\newcommand{\sigmav}{\sigma^2_{\bbv}}
\newcommand{\sigmaw}{\sigma^2_{\bbw}}
\newcommand{\ccalLhat}{\widehat{\ccalL}}
\newcommand{\Lhat}{\widehat{\ccalL}}
\newcommand{\rhow}{\rho_{\bbw}}
\newcommand{\rhov}{\rho_{\bbv}}
\usepackage[T1]{fontenc}

\newcommand\numberthis{\addtocounter{equation}{1}\tag{\theequation}}
\usepackage{tabulary}

\usepackage[hidelinks]{hyperref}
\definecolor{darkred}{RGB}{150,0,0}
\definecolor{darkgreen}{RGB}{0,150,0}
\definecolor{darkblue}{RGB}{0,0,150}
\hypersetup{colorlinks=true, linkcolor=darkred, citecolor=darkblue, urlcolor=darkblue}

\newcommand\blfootnote[1]{%
  \begingroup
  \renewcommand\thefootnote{}\footnote{#1}%
  \addtocounter{footnote}{-1}%
  \endgroup
}

\title{\textbf{An Optimal Transport Approach to \\ Personalized Federated Learning}}

\date{}

\makeatletter
\renewcommand*{\@fnsymbol}[1]{\ensuremath{\ifcase#1\or 1 \or 2 \or 3 \or 4 \else\@ctrerr\fi}}
\makeatother
    
\newcommand*\samethanks[1][\value{footnote}]{\footnotemark[#1]}

\author{
Farzan~$\text{Farnia}^\dagger$\thanks{The Chinese University of Hong Kong, \{farnia@cse.cuhk.edu.hk\}.}, Amirhossein~$\text{Reisizadeh}^\dagger$\thanks{Massachusetts Institute of Technology, \{amirr@mit.edu, jadbabai@mit.edu\}.}, Ramtin~Pedarsani\thanks{University of California, Santa Barbara, \{ramtin@ece.ucsb.edu\}.}, 
Ali~Jadbabaie\samethanks[2]
	}

\usepackage{enumitem}
\usepackage[title]{appendix}
\usepackage[utf8]{inputenc} 
\usepackage[T1]{fontenc}    
\usepackage{url}            
\usepackage{booktabs}       
\usepackage{amsfonts}       
\usepackage{nicefrac}       
\usepackage{microtype}      
\usepackage{hhline}

\usepackage[T1]{fontenc}
\usepackage[utf8]{inputenc}
\usepackage{graphicx}
\usepackage{subcaption}
\usepackage{color}
\usepackage{amsmath}
\usepackage{bbm}
\usepackage{tcolorbox}
\usepackage{lipsum}  
\usepackage{tcolorbox}
\usepackage{amsfonts,amssymb}
\usepackage{mathtools}
\usepackage{commath}
\usepackage{relsize}
\usepackage{comment,color,soul}
\usepackage{amsfonts}
\usepackage{url}
\usepackage{lipsum}
\usepackage{nicefrac}
\usepackage{float}
\usepackage{titling}

\DeclareMathAlphabet{\mathsfit}{\encodingdefault}{\sfdefault}{m}{sl}
\SetMathAlphabet{\mathsfit}{bold}{\encodingdefault}{\sfdefault}{bx}{n}

\allowdisplaybreaks

\begin{document}

\maketitle

\begin{abstract}
Federated learning is a distributed machine learning paradigm, which aims to train a model using the local data of many distributed clients. A key challenge in federated learning is that the data samples across the clients may not be identically distributed. To address this challenge, personalized federated learning with the goal of tailoring the learned model to the data distribution of every individual client has been proposed. In this paper, we focus on this problem and propose a novel personalized Federated Learning scheme based on Optimal Transport (\texttt{FedOT}) as a learning algorithm that learns  the optimal transport maps for transferring data points to a common distribution as well as  the prediction model under the applied transport map. To formulate the \texttt{FedOT} problem, we extend the standard optimal transport task between two probability distributions to multi-marginal optimal transport problems with the goal of transporting samples from multiple distributions to a common probability domain. We then leverage the results on multi-marginal optimal transport problems to formulate \texttt{FedOT} as a min-max optimization problem and analyze its generalization and optimization properties. We discuss the results of several numerical experiments to evaluate the performance of \texttt{FedOT} under heterogeneous data distributions in federated learning problems.
\blfootnote{$\dagger$ Contributed equally.}
\blfootnote{The paper's code is accessible at the GitHub repository \href{https://github.com/farzanfarnia/FedOT}{https://github.com/farzanfarnia/FedOT.}}
\end{abstract}

\section{Introduction}
The proliferation of mobile devices requires learning algorithms capable of training a prediction model using data distributed across local users in a network. Federated learning \cite{mcmahan2017communication} is a recent learning paradigm where several users are connected to a central server and train a machine learning model through their communications with the server. While standard federated learning algorithms perform successfully under identically distributed training data at different users, this assumption does not usually hold in practical federated learning settings in which the training samples are collected by multiple agents with different backgrounds, e.g. speech and text data gathered from a multi-lingual community. To address the heterogeneity of users' data distributions, federated learning under heterogeneous data has received great attention in the machine learning community \cite{zhao2018federated,li2018federated,kairouz2019advances,li2019convergence,karimireddy2019scaffold}.

A recently studied approach for federated learning under non-identically distributed data is to adapt the globally trained model to the particular distribution of every local user. Based on this approach, instead of learning a common model shared by all the users, the learning algorithm tailors the trained model to the samples observed by every user in the network. 
As such personalized federated learning algorithms lead to different trained models at different users,  an important baseline for their evaluation is a locally-performing learning algorithm in which every user fits a separate model to only her own data. Therefore, the conditions under which the users can improve upon such a non-federated purely local baseline play a key role in the design of a successful  personalized federated learning method.

In a general federated learning setting with arbitrarily different users' distributions, the users do not necessarily benefit from cooperation through federated learning. For example, if the users aim for orthogonal classification objectives, their cooperation according to standard federated learning algorithms can even lead to worse performance than their locally trained models. To characterize conditions under which a mutually beneficial cooperation is feasible, a standard assumption in the literature is to bound the distance between the distributions of different users. However, such assumptions on the closeness of the distributions raise the question of whether federated learning will remain beneficial if the users' distributions do not stay in a small distance from each other.

In this work, we study the above question through the lens of optimal transport theory and demonstrate that a well-designed federated learning algorithm can still improve upon the users' locally-trained models as long as the transportation maps between the users' distributions can be properly learned from the training data. We show that this condition relaxes the bounded distance assumption used in the literature and further applies to any federated learning setting where the learners only have some rudimentary knowledge of the statistical nature of distribution shifts, e.g. under affine convolutional filters applied to change the color, brightness, and intensity of image data.   

To learn the personalized models under the above condition, we introduce \texttt{FedOT} as a \emph{Federated learning framework based on Optimal Transport}. According to \texttt{FedOT}, the users simultaneously learn the transportation maps for transferring their samples to a common probability domain and fit a global classifier to the transferred training data. To personalize the globally trained model to the specific distribution of every user, \texttt{FedOT} combines the global classifier with the learned transportation maps needed for transferring samples from the original distributions of local users to the common distribution. 


In order to formulate and solve \texttt{FedOT}, we leverage optimal transport theory to reduce \texttt{FedOT}'s learning task to a min-max optimization problem. To this end, we focus on an extension of standard optimal transport problems between two probability domains to a structured multi-marginal optimal transport task for mapping several different distributions to a common probability domain. 
{ In  Section \ref{sec:multi-marginal}, we review several key definitions and results from multi-marginal optimal transport theory for which we provide a unified set of notations and novel proofs. We generalize standard duality results in optimal transport theory to the multi-marginal setting, which results in a min-max formulation of \texttt{FedOT}. The main results in this section not only guide us toward formulating a minimax optimization problem for the \texttt{FedOT} framework (Theorem \ref{Thm: Kantorovich duality}), but also provide intuition on how to design the function spaces in the \texttt{FedOT} minimax approach (Theorem \ref{Thm: Brenier thm}). Specifically, we leverage the intuition offered by Theorem \ref{Thm: Brenier thm} to reduce the size of function spaces in the \texttt{FedOT} minimax problem and improve the generalization and optimization performance of the \texttt{FedOT} learners.}

 Next, we show that \texttt{FedOT}'s min-max formulation is capable of being decomposed into a distributed form, and thus \texttt{FedOT} provides a scalable federated learning framework. We further analyze the generalization and optimization properties of the proposed \texttt{FedOT} approach. Under the condition that the sample complexity of learning the classifier dominates the complexity of finding the transportation maps, we prove that \texttt{FedOT} enjoys a better generalization performance in comparison to locally trained models. In addition, we show that the formulated min-max optimization problem can be solved to a stationary min-max solution by a standard distributed gradient descent ascent (GDA) algorithm. Therefore, the min-max formulation leads to a tractable distributed optimization problem, since the iterative GDA updates can be decomposed into a distributed form. 

Finally, we discuss the results of our numerical experiments comparing the performance of \texttt{FedOT} with several standard federated learning schemes. Our experimental results demonstrate the success of \texttt{FedOT} under various types of distribution changes including affine distribution shifts and image color transformations. We can summarize the main contributions of this work as follows:
\setlist{leftmargin=5.5mm}
\begin{itemize} \setlength\itemsep{0em}
    \item Introducing \texttt{FedOT} as an optimal transport-based framework to the federated learning problem under heterogeneous data,
    \item Extending standard results of optimal transport theory to the multi-marginal optimal transport problem with the goal of transporting the input distributions to a common probability domain,
    \item Analyzing the generalization and optimization properties of \texttt{FedOT} and establishing conditions under which \texttt{FedOT} improves upon locally-learned models,
    \item Demonstrating the efficacy of \texttt{FedOT} through several numerical experiments on standard image recognition datasets and neural network architectures. 
\end{itemize}


\textbf{Related Work on Federated Learning and Min-Max Optimization.} There has been a vast variety of tools and techniques used to address the prersonalization challenge in federated learning. As discussed before, utilizing only a shared global model for all the clients fails to capture the discrepancies in users' data distributions. On the other hand, local models would not benefit from the samples of other clients if a mere local training is implemented. Therefore, a combination of the two trained models, global and local ones, would naturally provide a degree of personalization \cite{deng2020adaptive,mansour2020three,hanzely2020federated} which is also known as model interpolation. 

{
Meta-learning-based approaches to federated learning under heterogeneous data distributions have been proposed by the related works \cite{smith2017federated,fallah2020personalized,jiang2019improving}. According to these approaches, a local and personalized model is adapted for each client by performing a few gradient steps on a common global model. This family of federated learning algorithms have been shown to be successful in handling unstructured distribution shifts where the learners have no prior knowledge of the structure of distribution shifts in the underlying network. On the other hand, the main focus of our proposed \texttt{FedOT} framework is on the learning scenarios where the learners have some prior knowledge of the type of distribution shifts. 
}

In a data interpolation approach to personalized federated learning \cite{deng2020adaptive,mansour2020three}, a local model is trained for each client by minimizing the loss over a mixture of local and global distributions.  \cite{liang2020think,collins2021exploiting} propose to learn a common representation for personalized federated learning. Similarly, \cite{shamsian2021personalized} develop a personalized federated learning approach through a group of hypernetworks to update the neural net classifier. While our work pursues a similar goal of learning a common representation, it introduces a novel minimax learning algorithm by leveraging optimal transport theory.
 
Cluster-based federated learning methods based on clustering users with similar underlying distributions have also been explored in several related works \cite{ghosh2019robust,xie2020multi,ghosh2020efficient} to overcome the challenge of heterogeneous data in federated learning. As another approach, \cite{li2021fedbn} propose applying local batch normalization to train personalized neural network classifiers. In a slightly different approach to handle the data heterogeneity challenge in federated learning, \cite{mohri2019agnostic,reisizadeh2020robust,deng2021distributionally} propose different min-max formulations to train robust models against non-i.i.d. samples. Aside its federated learning applications, nonconvex-concave min-max optimization and its complexity guarantees have been extensively studied in the  literature \cite{lin2019gradient, yang2020global, nouiehed2019solving,deng2021local}.

\textbf{Related Work on Optimal Transport Frameworks in Machine Learning.} A large body of related works apply optimal transport theory to address various statistical learning problems. These applications include generative adversarial networks (GANs) \cite{arjovsky2017wasserstein,sanjabi2018convergence,feizi2020understanding}, distributionally robust supervised learning \cite{lee2017minimax,kuhn2019wasserstein,blanchet2019robust}, learning mixture models \cite{kolouri2018sliced,balaji2019normalized}, and combining neural network models \cite{singh2020model}. Multi-marginal optimal transport costs \cite{pass2015multi} have also been studied in other machine learning contexts including GANs \cite{cao2019multi}, domain adaptation \cite{hui2018unsupervised}, and Wasserstein barycenters \cite{cuturi2014fast,claici2018stochastic,kroshnin2019complexity}.

\section{Multi-input Optimal Transport Problems}\label{sec:multi-marginal}

{A useful approach to learning under heterogeneous data distributions is to transport the different input distributions to a shared probability domain and then learn a supervised learning model for the shared probability domain. This task can be cast as a multi-input optimal transport problem, since the goal is to map the input distributions to a common distribution. In this section, we review the key definitions and tools from multi-input optimal transport theory to address the transportation task. The results in this section guide us toward formulating a minimax optimization problem for federated learning under heterogeneous distributions, and further help to reduce the statistical and computational complexities of the learning problem through leveraging prior knowledge of the structure of distribution shifts in the federated learning setting.}

In the literature, the optimal transport problem is typically defined for transporting samples between two probability domains \cite{villani2009optimal}. For a cost function $c(x,x')$ measuring the cost of transporting $x$ to $x'$, optimal transport cost $W_c(P,Q)$ is defined through finding the coupling that leads to the minimum expected cost of transporting samples between $P,\, Q$:
\begin{equation}
    W_c(P,Q) := \min_{\pi\in\Pi(P,Q)}\mathbb{E}_{(X,X')\sim\pi}\bigl[c(X,X')\bigr]. \nonumber
\end{equation}
Here $\Pi(P,Q)$ denotes the set of all joint distributions on $(X,X')$ that are marginally distributed as $P$ and $Q$. Note that the above optimal transport cost quantifies the optimal expected cost of mapping samples between the domains $P$ and $Q$.

However, for several problems of interest in machine learning one needs to extend the above definition to multi-input cost functions where the goal is to transport samples across multiple distributions. To define the $n$-ary optimal transport cost, a standard extension \cite{pass2015multi} is to consider an $n$-ary cost function $c(x_1,\cdots,x_n)$ and define the $n$-ary optimal transport map as:
\begin{equation}
    W_c(P_1,\cdots,P_n):= \min_{\pi\in\Pi(P_1,\cdots,P_n)}\mathbb{E}_{\pi}\bigl[c(X_1,\cdots , X_n) \bigr], \nonumber
\end{equation}
where $\Pi(P_1,\cdots,P_n)$ denotes the set of joint distributions on $(X_1,\ldots , X_n)$ that are marginally distributed as $P_1,\ldots,P_n$, respectively.

Inspired by the personalized federated learning problem where our goal is to map the different input distributions to a common probability domain, we focus on the following type of $n$-ary cost functions throughout this paper, which is also referred to as the infimal convolution cost \cite{pass2015multi}. The optimal transport costs resulting from the following type of $n$-ary costs preserve the key features of standard optimal transport costs with binary cost $\tilde{c}(x,x')$:
\begin{equation}\label{Eq: n_ary cost}
    c(x_1,\cdots , x_n) = \min_{x'}\:\sum_{i=1}^n \tilde{c}(x',x_i).
\end{equation}
Such an $n$-ary cost function lets us focus on $n$-ary transportation problems where the goal is to transport all the $n$ inputs to a single point that minimizes the total cost of transportation. The following proposition by \cite{carlier2010matching} connects the $n$-ary optimal transport costs to binary optimal transport costs.
\begin{prop}[\cite{carlier2010matching}, Prop. 3]\label{Proposition: n-ary Wasserstein to standard}
Consider the $n$-ary cost in \eqref{Eq: n_ary cost}. Then,
\begin{equation}\label{Eq: n_ary W cost equals min sum W}
    W_c(P_1,\cdots,P_n) = \min_Q\: \sum_{i=1}^n W_{\tilde{c}}(Q,P_i).
\end{equation}
\end{prop}
\begin{proof}
We defer the proof to the Appendix.
\end{proof}
We note that if the binary cost function is chosen as a powered norm difference $\tilde{c}(\mathbf{x},\mathbf{x}')=\Vert \mathbf{x}-\mathbf{x}'\Vert^q$, then the proposed multi-marginal optimal transport cost simplifies to the well-known family of Wasserstein barycenters. 
Next, we present a generalization of the Kantorovich duality theorem to $n$-ary optimal transport costs with the characterized cost function. This result has been already shown in the optimal transport theory literature \cite{carlier2010matching}, and we present our new proof of the result in the Appendix. In the following theorem, we use the standard definition of the $c$-transform of a real-valued function $\phi$ as 
$
    \phi^{\tilde{c}}(x) := \min_{x'}\:  \tilde{c}(x,x') + \phi(x').
$
\begin{theorem}\label{Thm: Kantorovich duality}
For the $n$-ary cost in \eqref{Eq: n_ary cost}, we have the following duality result where each variable $\phi_i:\mathbb{R}^d\rightarrow\mathbb{R}$ denotes a real-valued function:
\begin{equation*}
    W_c(P_1,\cdots,P_n) = \max_{\substack{\phi_{1:n}:\\
    \forall \mathbf{x}:\, \sum_i \phi_i(\mathbf{x})=0
    }}\; \sum_{i=1}^n \mathbb{E}_{P_i}\bigl[\,\phi^{\tilde{c}}_i(\mathbf{X})\, \bigr].
\end{equation*}
\end{theorem}
\begin{proof}
We defer the proof to the Appendix.
\end{proof}
In above and henceforth, we use the short-hand notation $a_{1:n} \coloneqq \{a_1, \cdots, a_n\}$, for $n$ vectors $a_1, \cdots, a_n$.  Next, we apply the above result to standard norm-based cost functions and simplify the dual maximization problem for these Wasserstein costs:
\begin{ex}\label{Example: 1-Wasserstein}
For the $1$-Wasserstein cost $c_1(\mathbf{x}_1,\cdots,\mathbf{x}_n)=\min_{\mathbf{x}'}\sum_{i}\Vert\mathbf{x}_i-\mathbf{x}' \Vert$, we have
    \begin{equation}
    W_{c_1}(P_1,\cdots,P_n) = \max_{\substack{\phi_{1:n}:\, \text{\rm 1-Lipschitz}\\
    \forall \mathbf{x}: \, \sum_i \phi_i(\mathbf{x})\le 0 
    }} \; \sum_{i=1}^n \mathbb{E}_{P_i}\bigl[\phi_i(\mathbf{X})\bigr]. \numberthis
\end{equation}
Note that in the special case $n=2$, the triangle inequality implies that $c_1(\mathbf{x}_1,\mathbf{x}_2)=\Vert \mathbf{x}_1 - \mathbf{x}_2\Vert$ which leads to standard 1-Wasserstein distance in the optimal transport theory literature \cite{villani2009optimal}.
\end{ex}
\begin{ex}\label{Example: 2-Wasserstein} 
For the $2$-Wasserstein cost $c_2(\mathbf{x}_1,\cdots,\mathbf{x}_n)=\min_{\mathbf{x}'}\sum_{i}\Vert\mathbf{x}_i-\mathbf{x}' \Vert_2^2$, we have
    \begin{align}\label{Eq: W2_nary_dual}
    W_{c_2}(& P_1,\cdots,P_n) = \max_{\substack{\phi_{1:n}:\, \text{\rm convex}\\
    \forall \mathbf{x}: \, \frac{1}{n}\sum_i \phi_i(\mathbf{x})\le\frac{1}{2}\Vert \mathbf{x}\Vert_2^2
    }}\;
    \sum_{i=1}^n \mathbb{E}_{P_i}\bigl[\frac{1}{2}\Vert \mathbf{X}\Vert^2 - \phi_i^\star(\mathbf{X})\bigr].
    \end{align}
    In the above, $\phi^{\star}$ denotes the Fenchel conjugate defined as $\phi^\star(\mathbf{x}) \coloneqq \sup_{\mathbf{x}'} \mathbf{x}^\top\mathbf{x}' - \phi(\mathbf{x}')$. For the special case $n=2$, one can see $c_2(\mathbf{x}_1,\mathbf{x}_2)=\frac{1}{2}\Vert \mathbf{x}_1 - \mathbf{x}_2 \Vert^2_2$ which results in the standard 2-Wasserstein distance in the literature \cite{villani2009optimal}.
\end{ex}
The next result shows that in the case of the $2$-Wasserstein cost the optimal potential function $\phi^*_{1:n}$ will transport samples to a common probability domain matching the distribution $Q^*$ in \eqref{Eq: n_ary W cost equals min sum W} with the optimal sum of Wasserstein costs to the input distributions. This result has been previously shown in \cite{carlier2010matching}, and we present a new proof in the Appendix.
\begin{theorem}\label{Thm: Brenier thm}
Suppose that $\phi^*_1,\cdots,\phi^*_n$ denote the optimal solutions to \eqref{Eq: W2_nary_dual} for 2-Wasserstein dual optimization problem. Then,
\begin{equation}
    \forall\, 1\le i,j\le n:\quad \nabla \phi^{*\star}_i (\mathbf{X}_i) \stackrel{\text{\rm dist}}{=} \nabla \phi^{*\star}_j (\mathbf{X}_j). \nonumber
\end{equation}
In the above, each $\mathbf{X}_i$ denotes the $i$th random variable distributed according to $P_i$ and $\stackrel{\text{\rm dist}}{=}$ means the two random variables share an identical distribution.
\end{theorem}
\begin{proof}
We defer the proof to the Appendix.
\end{proof}
As implied by the above theorem, the gradients of optimal potential functions lead to transportation maps for transporting samples from the different input distributions to a common probability domain. As we discuss later, transporting input samples to a common probability distribution can help to reduce the generalization error of a distributed learning task.

\section{\texttt{FedOT}: Federated Learning based on Optimal Transport} \label{sec: FedOT}
\subsection{Federated Learning Setting}
We focus on a federated learning scenario with $n$ local nodes connected to a single parameter server. We assume that every node $i \in [n]$ observes $m$ training samples $\{(\mathbf{x}_{i,j},y_{i,j})\}_{j=1}^m$ which are independently sampled from distribution $P_i$. Note that the input distributions are in general different, leading to a non-i.i.d. federated learning problem.

To model the heterogeneity of the distributions across the network, we suppose that for each node $i$, there exists an invertible transportation map $\psi_i: \mathbb{R}^d\rightarrow \mathbb{R}^d$ that maps a sample $(\mathbf{X}_i,Y_i)$ observed by node $i$ to a common distribution, i.e.,
\begin{equation}
    \forall \,1\le i,j\le n: \;\; \bigl(\psi_i(\mathbf{X}_i),Y_i\bigr)\stackrel{\tiny\text{\rm dist}}{=} \bigl(\psi_j(\mathbf{X}_j),Y_j\bigr). \nonumber
\end{equation}
In the above, $\stackrel{\tiny\text{\rm dist}}{=}$ denotes an identical probability distribution for the transported samples. Therefore, the mappings $\psi_{1:n}$ transfer the input distributions across the network to a common probability domain. Furthermore, we assume that there exists a space of functions $\Psi=\{\psi_{\boldsymbol{\theta}}: \boldsymbol{\theta}\in\Theta\}$ parameterized by  $\boldsymbol{\theta}$ containing the underlying transportation map $\psi_i$'s in our described federated learning setting. 

In the above federated learning setting, one can simplify the  federated learning problem to finding a prediction rule $f_{\mathbf{w}}\in \ccalF$ which predicts label $Y$ from the transported data vector in the shared probability domain of $\psi_i(X_i)$'s. Here $\ccalF=\{f_{\mathbf{w}}: \, \mathbf{w}\in\mathcal{W} \}$ is the set of models for training the prediction rule parameterized by the vector $\bbw$. Since $\psi_i(X_i)$'s are identically distributed across the network, the collected transported samples from \emph{all} the nodes can be used to train the prediction rule $f_{\mathbf{w}}$. Note that after finding the optimal classification rule $f_{\mathbf{w}^*}$, every node $i$ can personalize the classification rule by combining the transportation function $\psi_i$ and $f_{\mathbf{w}^*}$. Here, the personalized classifier for node $i$ will be $f_{\mathbf{w}^*}(\psi_i(\cdot))$. 

{
\begin{remark}
According to the Brenier's theorem \cite{villani2009optimal,mccann2011five}, the existence of the invertible transportation maps $\psi_i: \mathbb{R}^d\rightarrow \mathbb{R}^d$ for $i=1,\ldots , n$ mapping client distribution $P_i$'s to a common domain is guaranteed under the regularity assumption that the input distributions are absolutely continuous with respect to one another. Furthermore, we note that our analysis requires this assumption  only  for the underlying client distributions and does not need the condition for the empirical distributions of training samples. 
\end{remark}
\begin{remark}
While the described setting requires the same marginal distribution $P_{Y}$ for every client's label variable $Y$, the optimal transport-based framework can be further extended to cases with heterogeneous marginal distributions. To do this, we need to extend the assumption on the clients' feature distribution $P_{\mathbf{X}}$  to the clients' conditional feature distribution $P_{\mathbf{X}|Y=y}$ for every label outcome $y\in\mathcal{Y}$. In the extended setting, we further assume that for every $y\in\mathcal{Y}$, invertible transportation map $\psi_{y,i}$'s exist such that the conditional feature distribution $P_{\psi_{y,i}(\mathbf{X}_i)|Y_i=y}$ is identical for different clients. In this work, our main focus is on the setting with heterogeneous feature distributions, as the gain attained by the optimal transport approach is obtained through leveraging the structures on the features distribution shifts. Nevertheless, we still note that the optimal transport approach can be further extended to learning settings with different marginal distributions on the label variable $Y$. 
\end{remark}
}

\subsection{FedOT as a Min-Max Optimization Problem}
In order to train a personalized classification rule $f_{\mathbf{w}}$ and transportation maps $\psi_{\boldsymbol{\theta}_{1:n}}$, we consider the following optimization problem:
\begin{align}
    \min_{\mathbf{w},\boldsymbol{\theta}_{1:n}}
    \widehat{\ccalL}(\bbw, \boldsymbol{\theta}_{1:n}),
    \,
    \text{s.t.}
    \,
    W_c\big(P_{\psi_{\boldsymbol{\theta}_1}({\mathbf{X}}_{1})},\cdots, P_{\psi_{\boldsymbol{\theta}_n}({\mathbf{X}}_{n})}\big)\le \eps. 
\end{align}
In the above problem, we denote the empirical risk under transport maps $\psi_{\boldsymbol{\theta}_{1:n}}$ as
\begin{align}
    \widehat{\ccalL}(\bbw, \boldsymbol{\theta}_{1:n})
    \coloneqq
    \frac{1}{mn}\sum_{i=1}^n\sum_{j=1}^m \, \ell\bigl(f_{\mathbf{w}}(\psi_{\boldsymbol{\theta}_i}(\mathbf{x}_{i,j})),y_{i,j} \bigr), \numberthis
\end{align}
which quantifies the empirical risk associated with the $mn$ transported data samples across the $n$ nodes and $W_c(\cdot,\cdots,\cdot)$ denotes the $n$-ary optimal transport cost which measures the distance among the input distributions. Ideally, one wants the $n$-ary optimal transport cost to take a zero value that is necessary for having the same probability distribution for different $\psi_{\boldsymbol{\theta}_i}(\mathbf{X}_i)$'s. However, due to the generalization error in estimating the optimal transport cost from finite training data we allow an $\epsilon$-bounded optimal transport cost in the above formulation. 

In our analysis, we transfer the constraint bounding the optimal transport cost to the objective via a Lagrangian penalty and study the following optimization problem for a non-negative constant $\lambda\ge 0$:
\begin{align}
    \min_{\mathbf{w},\boldsymbol{\theta}_{1:n}}\; \widehat{\ccalL}(\bbw, \boldsymbol{\theta}_{1:n})
    +
    \lambda W_c\bigl(P_{\psi_{\boldsymbol{\theta}_1}({\mathbf{X}}_{1})},\cdot\cdot , P_{\psi_{\boldsymbol{\theta}_n}({\mathbf{X}}_{n})}\bigr). 
\end{align}

In order to solve the above optimization problem, we apply the generalized Kantorovich duality in Theorem \ref{Thm: Kantorovich duality} and reduce the above optimization problem to a min-max optimization task:
\begin{align}\label{Eq: FedOT Min_max general}
    \min_{\mathbf{w},\boldsymbol{\theta}_{1:n}}\max_{\substack{\phi_{1:n}:\\
   \forall \mathbf{x}:\,\sum_i \phi_i(\mathbf{x})=0}} \widehat{\mathcal{L}}(\mathbf{w},\boldsymbol{\theta}_{1:n},\phi_{1:n}):=\frac{1}{mn}\sum_{i=1}^n\sum_{j=1}^m  \ell\bigl(f_{\mathbf{w}}(\psi_{\boldsymbol{\theta}_i}(\mathbf{x}_{i,j})),y_{i,j} )\bigr) 
   + \lambda  \phi^{\tilde{c}}_i(\psi_{\boldsymbol{\theta}_i}(\mathbf{x}_{i,j})).
\end{align}
We call the above min-max framework \emph{Federated Learning based on Optimal Transport (\texttt{FedOT})}. We note that \texttt{FedOT} represents a family of federated learning algorithms for different cost functions.   

To solve the above min-max problem of \texttt{FedOT} for neural network function variables $\phi_{1:n}$, we enforce the zero sum condition in the above problem through constraining every neural net in $\phi_{1:n}$ to share the same weights for all the layers before the last layer and satisfy a zero summation of the weights of the last layers. Here, for activation function $\rho(\cdot)$ and weight matrices $\mathbf{U}:=[U_1,\ldots,U_L],$ we let $\phi_{\mathbf{U}}$ represent the neural network's mapping to the last layer and $\mathbf{v}_{1:n}$  stand for the weights of the last layers with a zero sum, i.e., $\sum_{i} \mathbf{v}_i =\mathbf{0}$, and hence we use the following function variables:
\begin{align}
    &\phi_i(\mathbf{x}):=\mathbf{v}^{\top}_i\phi_{\mathbf{U}}(\mathbf{x}),\quad \phi_{\mathbf{U}}(\mathbf{x}) := \rho(U_L\rho(\cdots\rho(U_1\mathbf{x})\cdots)\\
    & \;\; \text{\rm s.t.}
    \quad \sum_{i=1}^n \mathbf{v}_i =\mathbf{0}. \nonumber
\end{align}
In the following, we characterize the \texttt{FedOT} learning problems for 1-Wasserstein and 2-Wasserstein cost functions as earlier defined in Examples \ref{Example: 1-Wasserstein} and \ref{Example: 2-Wasserstein}.
\begin{ex}\label{Example: 1-FedOT}
Consider the \texttt{FedOT} problem with the $1$-Wasserstein cost in Example \ref{Example: 1-Wasserstein}. This formulation with neural net $\phi_i$'s leads to the $1$-FedOT min-max problem:
\begin{align}\label{Eq: 1-FedOT}
    \min_{\mathbf{w},\boldsymbol{\theta}_{1:n}}\;\max_{\substack{\mathbf{v}_{1:n}, \bbU:\\ \mathbf{v}_i^\top \phi_{\mathbf{U}}\, \text{1-Lipschitz},\atop
   \sum_i \mathbf{v}_i=\mathbf{0}
   }} \;  
   \frac{1}{mn}\sum_{i=1}^n\sum_{j=1}^m \,\biggl[
   \ell\bigl(f_{\mathbf{w}}(\psi_{\boldsymbol{\theta}_i}(\mathbf{x}_{i,j})),y_{i,j} \bigr) 
   + \lambda \mathbf{v}^T_i \phi_{\mathbf{U}}(\psi_{\boldsymbol{\theta}_i}(\mathbf{x}_{i,j}))\biggr].
\end{align}
\end{ex}
\begin{ex}\label{Example: 2-FedOT}
Consider the \texttt{FedOT} problem with the $2$-Wasserstien cost in Example \ref{Example: 2-Wasserstein}. This formulation leads to the $2$-FedOT min-max problem:
\begin{align}\label{Eq: 2-FedOT}
    \min_{\mathbf{w},\boldsymbol{\theta}_{1:n}}&\max_{\substack{\mathbf{v}_{1:n}, \bbU:\\ \mathbf{v}_i^\top \phi_{\mathbf{U}}\, \text{1-convex},\\
   \sum_i \mathbf{v}_i=\mathbf{0}
   }}   
   \frac{1}{mn}\sum_{i=1}^n\sum_{j=1}^m\biggl[ \ell\bigl(f_{\mathbf{w}}(\psi_{\boldsymbol{\theta}_i}(\mathbf{x}_{i,j})),y_{i,j} \bigr)  \!+\! \frac{\lambda}{2}\Vert\psi_{\boldsymbol{\theta}_i}(\mathbf{x}_{i,j})\Vert^2
   \!-\! \lambda \bigl(\mathbf{v}^\top_i \phi_{\mathbf{U}}\bigr)^\star(\psi_{\boldsymbol{\theta}_i}(\mathbf{x}_{i,j}))\biggr].\numberthis
\end{align}
Here, a function $g(\mathbf{x})$ is called $1$-convex if $g(\mathbf{x})+\frac{1}{2}\Vert \mathbf{x} \Vert^2_2$ is a convex function. Also, $(\mathbf{v}^\top_i \phi_{\mathbf{U}})^\star$ denotes the Fenchel conjugate of $\mathbf{v}^\top_i \phi_{\mathbf{U}}$.
\end{ex}
Next, we reduce \eqref{Eq: 2-FedOT} to an $L_2$-regularized min-max optimization problem with no Fenchel conjugates.
\begin{prop}\label{Proposition: FedOT Regularized}
Suppose that the maximization variables in \eqref{Eq: 2-FedOT} are constrained such that $\mathbf{v}^\top_i\phi_{\mathbf{U}}$ is $\gamma$-smooth, i.e., $\nabla_{\mathbf{x}}\mathbf{v}^\top_i\phi_{\mathbf{U}}(\mathbf{x})$ is $\gamma$-Lipschitz w.r.t. $\mathbf{x}$, and the operator norm of every layer of neural net $\phi_{\mathbf{U}}$ satisfies $\Vert U_i\Vert_2\le 1$. Then, the min-max objective in \eqref{Eq: 2-FedOT} is lower-bounded by:
\begin{align}\label{Eq: 2-FedOT Regularized}
   \frac{1}{mn}\sum_{i=1}^n\sum_{j=1}^m \biggl[\, \ell\bigl(f_{\mathbf{w}}(\psi_{\boldsymbol{\theta}_i}(\mathbf{x}_{i,j})),y_{i,j} \bigr)  
   + \lambda \mathbf{v}^\top_i \phi_{\mathbf{U}}\bigr(\psi_{\boldsymbol{\theta}_i}(\mathbf{x}_{i,j})) -\frac{\lambda}{1-\gamma}\bigl(\Vert\mathbf{v}_i\Vert_2^2+\Vert\mathbf{U}\Vert_F^2\bigr)\biggr],
\end{align}
where $\Vert\mathbf{U}\Vert_F$ denotes the Frobenius norm of $\mathbf{U}=[U_1,\ldots,U_L]$ defined as $\Vert\mathbf{U}\Vert^2_F:=\sum_{i=1}^L\Vert U_i\Vert^2_F$.
\end{prop}
\begin{proof}
We defer the proof to the Appendix.
\end{proof}
Note that if $\mathbf{v}^\top_i \phi_{\mathbf{U}}\bigr(\psi_{\boldsymbol{\theta}_i}(\mathbf{x}))$ is $\gamma'$-smooth as a function of $\mathbf{v}_i,\mathbf{U}$ where $\gamma'<\frac{1}{1-\gamma}$, then the min-max objective in \eqref{Eq: 2-FedOT Regularized} will be $\lambda\bigl(\frac{1}{1-\gamma}-\gamma'\bigr)$-strongly concave in terms of the maximization variables, resulting in a nonconvex strongly-concave min-max problem. We later show a federated gradient descent ascent (GDA) algorithm can solve such a min-max problem to find a first-order stationary min-max solution.   


\section{Generalization and Optimization Properties of \texttt{FedOT}} \label{sec: guarantees}
\subsection{Generalization Guarantees}
As discussed in the previous section, \texttt{FedOT} formulates the federated learning problem through the min-max optimization problem in \eqref{Eq: FedOT Min_max general}. In the heterogeneous case where every agent $i$ observes samples drawn from a different distribution $P_i$, the min-max objective of \eqref{Eq: FedOT Min_max general} provides an empirical estimation of the following true min-max objective:
\begin{align}
     \mathcal{L}(\mathbf{w},\boldsymbol{\theta}_{1:n},\phi_{1:n})  :=\, \frac{1}{n}\sum_{i=1}^n\mathbb{E}_{P_i}\bigl[\, \ell\bigl(f_{\mathbf{w}}(\psi_{\boldsymbol{\theta}_i}(\mathbf{X})),Y \bigr)  + \lambda \phi^{\tilde{c}}_i(\psi_{\boldsymbol{\theta}_i}(\mathbf{X}))\,\bigr]. 
\end{align}
With no assumptions on the optimal transport functions, estimating the above objective for all $\phi_i$'s will require an exponentially growing number of training samples in the dimension of data variable $\mathbf{X}$ \cite{panaretos2019statistical}. In order to mitigate such an exponential complexity, we assume that for any feasible underlying $P_i$, the optimal potential functions $\phi^*_{1:n}$ belong to a set of functions $\Phi$ with bounded complexity. Under the assumption that for all feasible $\theta_i$'s, $\phi^*_i\in\Phi$ is satisfied for optimal $\phi_i$'s one can equivalently solve the min-max problem \eqref{Eq: FedOT Min_max general} with the additional constraints $\forall i:\: \phi_i\in\Phi$, which as will be shown attains a bounded generalization error.

In our generalization analysis, we use the following standard definition of the covering number $\mathcal{N}(\mathcal{F} ,\epsilon,\Vert\cdot\Vert_{\infty})$ of a set of functions $\mathcal{F}$ with respect to the $L_\infty$-norm:
\begin{align*}
   \mathcal{N}(\mathcal{F} ,\epsilon,\Vert\cdot\Vert_{\infty}) := \min\bigl\{ N \in\mathbb{N}: \text{\rm an $\epsilon$-covering of}\; \mathcal{F} \: \text{\rm exists w.r.t. } \Vert\cdot\Vert_{\infty}\, \text{\rm with size } N\bigr\}. 
\end{align*}
In order to simplify our theoretical statements, we use the following notation in our theorems where $M:=\sup_{f\in\mathcal{F}, \mathbf{x}\in\mathcal{X}} f(\mathbf{x})$ and 
$
    \mathcal{V}(\mathcal{F}):= \int_{0}^1 \sqrt{\log\mathcal{N}(\mathcal{F} ,M\epsilon,\Vert\cdot\Vert_{\infty})}\,d\epsilon.
$
\begin{theorem}
Suppose that the loss function $\ell$ is $L_{\ell}$-Lipschitz and the expected loss is bounded by $M$ under all feasible distributions. Assume that for any $\mathbf{w}\in\mathcal{W}$, $\phi\in\Phi$, $\boldsymbol{\theta}\in\Theta$, $f_\mathbf{w}$, $\phi$, $\psi_{\boldsymbol{\theta}}$ are $L_{\bbw}$, $L_\phi$, $L_{\theta}$-Lipschitz.  Then, $ \forall \delta >0$ with probability at least $1-\delta$ the following holds for all $ \mathbf{w}\in\mathcal{W}$ in \eqref{Eq: 1-FedOT}
\begin{align*}
    &\quad \biggl\vert \min_{\boldsymbol{\theta}_{1:n}}\max_{\substack{\phi_{1:n}\in  \Phi:\atop
   \forall \mathbf{x}:\;\sum_i \phi_i(\mathbf{x})=0
   }}
   \mathcal{L}(\mathbf{w},\boldsymbol{\theta}_{1:n},\phi_{1:n}) -
   \min_{\boldsymbol{\theta}_{1:n}}\max_{\substack{\phi_{1:n}\in  \Phi:\atop
   \forall \mathbf{x}:\;\sum_i \phi_i(\mathbf{x})=0
   }}
   \widehat{\mathcal{L}}(\mathbf{w},\boldsymbol{\theta}_{1:n},\phi_{1:n}) \biggr\vert\\
    &\leq
    \mathcal{O}\Biggl({L_{\ell} L_{\bbw} M}\sqrt{\frac{\bigl(\mathcal{V}(\mathcal{W})+\mathcal{V}(\Theta)\big)^2\log(1/\delta)}{mn}} +
    {\lambda L_\phi L_\theta M} \sqrt{\frac{\big(\mathcal{V}(\Phi)+\mathcal{V}(\Theta)\big)^2\log(n/\delta)}{m}}+\frac{ML_w L_\ell}{\lambda} \Biggr).
\end{align*}
\end{theorem}
\begin{proof}
We defer the proof to the Appendix.
\end{proof}
The above theorem suggests that the sample complexity will scale linearly with $mn$, which is the total number of samples observed in the network, under the condition that $\mathcal{V}(\Phi)+\mathcal{V}(\Theta)< \frac{\mathcal{V}(\mathcal{\mathcal{W})}}{n}$, i.e., if the complexity measure of the classifier function space $\mathcal{W}$ is lower-bounded by the product of the number of users and the total complexity measure of $\Phi$ and $\Theta$.

\subsection{Optimization Guarantees}

\begin{algorithm*}
\textbf{Initialize} initial models $(\bbw_0, \bbv_0)$, stepsizes $\eta_1, \eta_2$, number of local updates $\tau$\\
\For{$t = 1, \cdots, T-1$}{
    \If{$t \nmid \tau$}{
    \vspace{-.3cm}
    \begin{align} 
        \bbw^i_{t+1} = \bbw^i_t - \eta_1 \sgrw \ccalLhat_i(\bbw^i_t, \bbv^i_t) \quad \text{and} \quad
        \bbv^i_{t+1} = \bbv^i_t + \eta_2 \sgrv \ccalLhat_i(\bbw^i_t, \bbv^i_t) 
    \end{align}
    \vspace{-.3cm}
    }
    \Else{
    \vspace{-.3cm}
    \begin{align} 
        \bbw^i_{t+1} = \frac{1}{n} \sum_{k=1}^{n} \left[ \bbw^k_t - \eta_1 \sgrw \ccalLhat_k(\bbw^k_t, \bbv^k_t) \right] \quad \text{and} \quad
        \bbv^i_{t+1} = \frac{1}{n} \sum_{k=1}^{n} \left[ \bbv^k_t + \eta_2 \sgrv \ccalLhat_k(\bbw^k_t, \bbv^k_t) \right]
    \end{align}
    \vspace{-.3cm}
    }
}
\textbf{Output} $\wbar_{T} \!= \! \frac{1}{n} \sum_{i=1}^{n} \bbw^i_T$ and $\vbar_{T}\! =\! \frac{1}{n} \sum_{i=1}^{n} \bbv^i_T$
\caption{\texttt{FedOT-GDA}}\label{alg: gda}
\end{algorithm*}\vspace{3mm}

To solve \texttt{FedOT} nonconvex-strongly-concave minimax problem \eqref{Eq: 2-FedOT Regularized}, we propose a gradient descent-ascent (GDA) method in Algorithm \ref{alg: gda}, namely \texttt{FedOT-GDA}, and further analyze its optimization properties. For the purpose of readability, we present our method and results using the following notation for  the minimax formulation:
\begin{equation} \label{eq: minimax w v}
    \min_{\bbw \in \ccalW} \max_{\bbv \in \ccalV} \ccalLhat(\bbw, \bbv) 
    \coloneqq 
    \frac{1}{n}\sum_{i=1}^{n} \ccalLhat_i(\bbw, \bbv),
\end{equation}
where each $\ccalLhat_i$ denotes the local loss function corresponding to node $i$'s samples. Here, $\bbw$ and $\bbv$ respectively denote the minimization and maximization variables described in  \eqref{Eq: 2-FedOT Regularized}, i.e. $\bbw = \{\bbw, \bbtheta_{1:n}\}$ and $\bbv = \{\bbv_{1:n}, \bbU\}$. We propose the following iterative GDA routine summarized in Algorithm \ref{alg: gda}. Let us denote by $(\bbw^i_t, \bbv^i_t)$ the local variable corresponding to node $i$ at iteration $t$. In every round, each node $i$ updates its local models $(\bbw^i_t, \bbv^i_t)$ using the stepsizes $\eta_1, \eta_2$ for $\tau$ successive iterations. Then, all updated local variables are uploaded to the parameter server and the corresponding averages are sent back to local nodes as the initial point for the next round of updates. There, $\sgrw \ccalLhat_i$ and $\sgrv \ccalLhat_i$ denote stochastic gradients of local losses w.r.t. their first and second arguments. It is important to note that \texttt{FedOT-GDA} imposes small communication (with periodic synchronization) and computation burden (by one gradient computation per iteration) on the network which is essential in federated learning methods.

As mentioned in Section \ref{sec: FedOT}, for smooth enough loss functions, the minimax objective in \eqref{Eq: 2-FedOT Regularized} is nonconvex-strongly-concave. That is, $\ccalLhat(\bbw, \bbv)$ in \eqref{eq: minimax w v} is nonconvex in $\bbw$ and strongly-concave in $\bbv$. The following set of assumptions formally characterizes the setting.
\vspace{-1mm}
\begin{assumption} \label{assumption: loss}
(i)  $\ccalV$ is a convex and bounded set with a diameter $D$. (ii) Local functions $\ccalLhat_i(\bbw, \bbv)$ have $L$-Lipchits gradients and are $\mu$-strongly concave in $\bbv$. That is, for both $* \in \{\bbw, \bbv\}$
\begin{align} 
    \Vert \gr_{*} \ccalLhat_i(\bbw, \bbv) - \gr_{*} \ccalLhat_i(\bbw', \bbv') \Vert^2 \leq
    L^2 \left( \Vert \bbw - \bbw'\Vert^2 + \Vert \bbv - \bbv'\Vert^2 \right).
\end{align}
We denote the condition number by $\kappa \coloneqq L / \mu$. (iii) (Gradient Diversity) There are constants $\rho_{\bbw}$ and $\rho_{\bbv}$ such that for both $* \in \{\bbw, \bbv\}$, we have that
$ 
    \frac{1}{n} \sum_{i=1}^{n} \Vert \gr_{*} \ccalLhat_i(\bbw, \bbv) - \gr_{*} \ccalLhat(\bbw, \bbv) \Vert^2 \leq \rho_{*}^2.
$
\end{assumption}

Since the global loss function $\ccalLhat(\bbw, \bbv)$ is nonconvex w.r.t. the minimization variable $\bbw$, we aim to find $\epsilon$-stationary solutions for the primal function $\Lambda(\bbw) \coloneqq \max_{\bbv \in \ccalV} \ccalLhat(\bbw, \bbv)$. Next theorem characterizes the convergence rate of the proposed \texttt{FedOT-GDA} in Algorithm \ref{alg: gda} to find a stationary solution for $\min_{\bbw \in \ccalW}\Lambda(\bbw)$.
\begin{theorem} \label{thm: opt}
Consider the iterates $\{\bbw^i_t, \bbv^i_t\}$ in Algorithm \ref{alg: gda} and let Assumption \ref{assumption: loss} hold. Moreover, assume that the local stochastic gradients are unbiased and variance bounded, i.e., $\E \Vert \sgr_{*}\ccalLhat_i (\bbw, \bbv) - \gr_{*} \ccalLhat_i(\bbw, \bbv) \Vert^2 \leq \sigma_{*}^2$ for $* \in \{\bbw, \bbv\}$. Then, there exists iteration $t \in \{0, \cdots, T-1\}$ for which
\begin{align} \label{eq: Thm opt}
    \E \Vert\gr \Lambda(\wbar_{t}) \Vert^2
    \leq
    \ccalO \Bigg( \frac{\Delta_{\Lambda}}{\eta_1 T}  
    +
    \frac{\kappa^3 L D^2}{\eta_2 T} 
    +
    \eta_1 \frac{\sigmaw}{n}
    +
    \eta_2 \kappa^2 L \frac{\sigmav}{n} 
    +
    \eta^2_{\sigma}\kappa^2 L^2 \tau 
    +
    \eta^2_{\rho}\kappa^2 L^2\tau^2 
    \Bigg),
\end{align}
where $\eta^2_{\sigma} \coloneqq \eta_1^2\sigmaw +\eta_2^2\sigmav$, $\eta^2_{\rho} \coloneqq \eta_1^2 \rho_{\bbw}^2+\eta_2^2 \rho_{\bbv}^2$ and $\Delta_{\Lambda} \coloneqq \Lambda(\bbw_0) - \min_{\bbw \in \ccalW} \Lambda(\bbw)$.
\end{theorem}
\begin{proof}
We defer the theorem's proof to the Appendix.
\end{proof}

The result of Theorem \ref{thm: opt} indicates that after $T$ iterations of \texttt{FedOT-GDA} in Algorithm \ref{alg: gda} and for proper choices of the stepsizes $\eta_1 = \ccalO(\nicefrac{1}{\sqrt{T}})$ and $\eta_2 = \ccalO(\nicefrac{1}{\sqrt{T}})$, an $\epsilon$-stationary solution $\wbar$ for the min-max problem \eqref{eq: minimax w v} (and hence \eqref{Eq: 2-FedOT Regularized}) can be obtained for which $\E \Vert\gr \Lambda(\wbar) \Vert^2 \leq \ccalO(\nicefrac{1}{\sqrt{T}})$. However, we still note that this result requires the inner maximization objective to be strongly-concave. Extending this result to general nonconvex-nonconcave settings is an interesting future direction to this work. 

\section{Numerical Results}

\begin{table*}[t]
 \centering
 {\renewcommand{\arraystretch}{1.05}%
\resizebox{\textwidth}{!}{ \begin{tabu}{|c|ccc|ccc|cc|}
 \hline
 Dataset & \multicolumn{3}{c|}{MNIST} & \multicolumn{3}{c|}{CIFAR-10} & 
 \multicolumn{2}{c|}{Colored-MNIST} \\ \hline
   Method        & $m\!=\!50$ & $m\!=\!100$ & $m\!=\!500$ & $m\!=\!50$ & $m\!=\!100$ & $m\!=\!500$ & $m\!=\!50$ & $m\!=\!500$    \\ \hhline{=|===|===|==|}
 \texttt{FedOT}, $\tau\!=\! 1$  & $\mathbf{87.0\%}$ & $\mathbf{95.6\%}$ & $\mathbf{97.0\%}$ & $\mathbf{42.2\%}$ & $\mathbf{51.6\%}$ & $61.8\%$ & $86.0\%$ & $96.6\%$ \\ 
 \texttt{FedOT}, $\tau\!=\! 5$  & $85.4\%$ & $94.4\%$ & $96.4\%$ & $40.8\%$ & $51.2\%$ & $\mathbf{63.0\%}$ & $\mathbf{88.6\%}$ & $\mathbf{97.4\%}$ \\
 \hline
 \texttt{FedAvg}, $\tau\!=\! 1$  & $72.0\%$ & $78.4\%$ & $86.8\%$ & $22.8\%$ & $26.4\%$ & $37.2\%$ & $75.8\%$ & $90.8\%$ \\
 \texttt{FedAvg}, $\tau\!=\! 5$  & $64.8\%$ & $72.6\%$ & $82.2\%$ & $18.8\%$ & $25.0\%$ & $36.6\%$ & $73.2\%$ & $91.4\%$ \\ 
 \hline
 \texttt{L-FedAvg}, $\tau\!=\! 1$ & $66.4\%$ & $74.2\%$ & $88.0\%$ & $17.8\%$ & $23.0\%$ & $39.0\%$ & $71.0\%$ & $91.2\%$ \\
 \texttt{L-FedAvg}, $\tau\!=\! 5$  & $61.2\%$  & $71.2\%$ & $85.0\%$ & $16.0\%$ & $22.6\%$ & $36.8\%$ & $69.8\%$ & $92.0\% $ \\ 
 \hline
 \texttt{FedMI}, $\tau\!=\! 1$  & $64.0\%$ & $75.8\%$ & $87.4\%$ & $21.0\%$ & $27.0\%$ & $40.2\%$ & $64.0\%$ & $91.8\%$ \\
 \texttt{FedMI}, $\tau\!=\! 5$  & $61.8\%$ & $74.0\%$ & $85.4\%$ & $17.6\%$ & $25.4\%$ & $37.0\%$ & $62.2\%$ & $92.6\%$ \\ 
 \hline
 \texttt{Fed-FOMAML}, $\tau\!=\! 1$  & $52.2\%$ & $81.0\%$ & $89.0\%$ & $14.8\%$ & $31.4\%$ & $46.4\%$ & $66.8\%$ & $94.6\%$ \\
\texttt{Fed-FOMAML}, $\tau\!=\! 5$  & $44.0\%$ & $77.8\%$ & $88.2\%$ & $12.0\%$ & $28.6\%$ & $45.6\%$ & $58.4\%$ & $94.2\%$ \\ 
 \hline
 
 \end{tabu}}}
 \vspace{2mm}
 \caption{{AlexNet results:} Average test accuracy under affine distribution shifts (MNIST \& CIFAR-10) and color transformations (Colored-MNIST) and different training set sizes per user $m$ computed for \texttt{FedOT} vs. the baseline methods including \texttt{FedAvg}, Local-\texttt{FedAvg} (\texttt{L-FedAvg}), Federated Model Interpolation (\texttt{FedMI}), and { Federated First-Order Model Agnostic Meta Learning \texttt{Fed-FOMAML}}.}
 \label{tbl:FedOT_AlexNet}
 \end{table*}



We evaluated the empirical performance of our proposed \texttt{FedOT} method on standard image recognition datasets including MNIST \cite{lecun1998gradient}, CIFAR-10 \cite{krizhevsky2009learning}, and Colored-MNIST \cite{arjovsky2019invariant}. We used the standard AlexNet \cite{krizhevsky2012imagenet} and InceptionNet \cite{szegedy2015going} neural network architectures in our experiments which we implemented in the TensorFlow platform \cite{abadi2016tensorflow}. For the federated learning setting, we used a network of $n=100$ users and ran every experiment with three user-based training size values: $m=50,100,500$. We also tested two values of $\tau=1,5$ for the number of local steps before every synchronization. In our experiments, we simulated the following two types of distribution shifts:
\begin{enumerate}[topsep=0pt,leftmargin=*]
    \item \textbf{Affine distribution shifts}: Here, we drew $n=100$ random isotropic Gaussian vectors $\mathbf{z}_i\sim\mathcal{N}(\mathbf{0},\sigma I_d)$ with $\sigma=1$ and $n$ random uniformly-distributed vectors $\mathbf{s}_i\sim \operatorname{Unif}([0.5,1.5]^d)$ and manipulated every training sample $\mathbf{x}_{i,j}$ at the $i$th node as follows
    \begin{equation}
        \forall i,j:\; \mathbf{x}'_{i,j}=\operatorname{diag}\{\mathbf{s}_i\}\mathbf{x}_{i,j} + \mathbf{z}_i. \nonumber
    \end{equation}
    \item \textbf{Color-based distribution shifts}: We experimented color-based shifts on MNIST samples. Here, we used a threshold of $\zeta = 10^{-4}$ to detect near-zero pixel values for every MNIST sample. Then, we drew $n$ pairs of uniformly-distributed vectors $\mathbf{a}_i,\mathbf{b}_i\in \operatorname{Unif}([0,1]^3)$ (corresponding to the three RGB channels) and manipulated every pixel $(l_1,l_2)$ as follows:
    \begin{equation}
        \forall i,j,l_1,l_2:\; \mathbf{x}'_{i,j,l_1,l_2}=\begin{cases} \mathbf{a}_i \quad &\text{\rm if}\; x_{i,j,l_1,l_2}\le \zeta, \\ 
        x_{i,j,l_1,l_2}\mathbf{b}_i\quad &\text{\rm if}\; x_{i,j,l_1,l_2} > \zeta.
        \end{cases} \nonumber
    \end{equation}
\end{enumerate}

\begin{table*}[t]
 \centering
 {\renewcommand{\arraystretch}{1.05}%
\resizebox{\textwidth}{!}{ \begin{tabu}{|c|ccc|ccc|cc|}
 \hline
 Dataset & \multicolumn{3}{c|}{MNIST} & \multicolumn{3}{c|}{CIFAR-10} & 
 \multicolumn{2}{c|}{Colored-MNIST} \\ \hline
   Method        & $m\!=\!50$ & $m\!=\!100$ & $m\!=\!500$ & $m\!=\!50$ & $m\!=\!100$ & $m\!=\!500$ & $m\!=\!50$ & $m\!=\!500$    \\ \hhline{=|===|===|==|}
 \texttt{FedOT}, $\tau\!=\! 1$  & $\mathbf{76.6\%}$ & $\mathbf{83.2\%}$ & $\mathbf{91.0\%}$ & $\mathbf{50.4\%}$ & $\mathbf{59.2\%}$ & $70.4\%$ & $\mathbf{77.4\%}$ & $\mathbf{97.0\%}$ \\ 
 \texttt{FedOT}, $\tau\!=\! 5$  & $73.0\%$ & $82.6\%$ & $90.6\%$ & $48.4\%$ & $57.8\%$ & $\mathbf{72.2\%}$ & $72.0\%$ & $96.6\%$ \\
 \hline
 \texttt{FedAvg}, $\tau\!=\! 1$  & $70.8\%$ & $78.8\%$ & $84.2\%$ & $29.2\%$ & $34.6\%$ & $44.0\%$ & $69.8\%$ & $89.8\%$ \\
 \texttt{FedAvg}, $\tau\!=\! 5$  & $66.2\%$ & $75.0\%$ & $83.4\%$ & $25.0\%$ & $32.8\%$ & $45.2\%$ & $67.2\%$ & $90.6\%$ \\ 
 \hline
 \texttt{L-FedAvg}, $\tau\!=\! 1$ & $67.4\%$ & $78.0\%$ & $84.6\%$ & $23.4\%$ & $32.8\%$ & $43.8\%$ & $65.4\%$ & $92.2\%$ \\
 \texttt{L-FedAvg}, $\tau\!=\! 5$  & $63.0\%$  & $76.8\%$ & $83.8\%$ & $19.6\%$ & $30.4\%$ & $46.6\%$ & $63.6\%$ & $92.4\%$  \\ 
 \hline
 \texttt{FedMI}, $\tau\!=\! 1$  & $58.2\%$ & $73.6\%$ & $82.6\%$ & $23.6\%$ & $33.6\%$ & $44.8\%$ & $61.4\%$ & $92.0\%$ \\
 \texttt{FedMI}, $\tau\!=\! 5$  & $54.6\%$ & $74.6\%$ & $83.2\%$ & $19.2\%$ & $34.0\%$ & $45.2\%$ & $59.8\%$ & $92.6\%$ \\ 
 \hline
 \texttt{Fed-FOMAML}, $\tau\!=\! 1$  & $58.0\%$ & $80.2\%$ & $86.6\%$ & $16.8\%$ & $34.0\%$ & $49.4\%$ & $67.0\%$ & $94.2\%$ \\
\texttt{Fed-FOMAML}, $\tau\!=\! 5$  & $46.2\%$ & $73.8\%$ & $86.0\%$ & $16.2\%$ & $32.6\%$ & $48.8\%$ & $65.6\%$ & $94.2\%$ \\
 \hline
 \end{tabu}}}
 \vspace{2mm}
 \caption{{InceptionNet results:} Average test accuracy under affine distribution shifts (MNIST \& CIFAR-10) and color transformations (Colored-MNIST) and different training set sizes per user $m$ computed for \texttt{FedOT} vs. the baseline methods including \texttt{FedAvg}, Local-\texttt{FedAvg} (\texttt{L-FedAvg}), Federated Model Interpolation (\texttt{FedMI}), and { Federated First-Order Model Agnostic Meta Learning \texttt{Fed-FOMAML}}.}  \label{tbl:FedOT_InceptionNet}
 \end{table*}

{
We use the insight offered by Theorem \ref{Thm: Brenier thm} to design the class of potential functions in these numerical experiments. As shown in Theorem \ref{Thm: Brenier thm}, the optimal potential function will also be the integral of the optimal transport maps which will be an affine transformation under affine distribution shifts and a piecewise affine transformation under color-based distribution shifts. Therefore, we used the following class of functions $\Phi$ and $\Theta$ in our experiments: 
} 
 
\begin{enumerate}[topsep=0pt,leftmargin=*]
    \item For affine distribution shifts, we applied affine transformations $\psi_{\boldsymbol{\theta}_{1:n}}$ and quadratic potential functions $\phi_{1:n}$, where $\forall i$,
   \begin{align}
         &\psi_{\boldsymbol{\theta}_i}(\mathbf{x}) = {\Theta}_{i,1}\mathbf{x} + \boldsymbol{\theta}_{i,0},\;\; \phi_{\mathbf{v}_i}(\mathbf{x}) = \frac{1}{2}\mathbf{x}^\top{V}_{i,0}\mathbf{x} + \mathbf{v}_{i,1}^\top\mathbf{x},\\
    &\quad \text{s.t.}
    \quad \sum_{i=1}^n {V}_{i,0} = \mathbf{0} \quad \text{and}\quad \sum_{i=1}^n \mathbf{v}_{i,1} = \mathbf{0}.
    \end{align}
    \item For color-based distribution shifts, we considered one-hidden layer neural networks with ReLU activation ($\operatorname{ReLU}(z)=\max\{z,0\}$) for both $\psi_{\boldsymbol{\theta}_{1:n}}$ and potential functions $\phi_{1:n}$, where
    \begin{align}
         \forall i: \psi_{\boldsymbol{\theta}_i}(\mathbf{x}) = \operatorname{ReLU}({\Theta}_{i,2}\mathbf{x} + \boldsymbol{\theta}_{i,1}) + \boldsymbol{\theta}_{i,0},\phi_{\mathbf{v}_i}(\mathbf{x}) = \mathbf{v}_{i,2}^\top \operatorname{ReLU}({V}_{1}\mathbf{x}+\mathbf{v}_{0}),\; \text{s.t.}
    \;\; \sum_{i=1}^n \mathbf{v}_{i,2} = \mathbf{0}.
    \end{align}
\end{enumerate}
We used the \texttt{FedOT-GDA} algorithm (Algorithm \ref{alg: gda}), that is a distributed mini-batch stochastic GDA, for solving the regularized \texttt{FedOT}'s min-max problem as formulated in Proposition \ref{Proposition: FedOT Regularized}. We used a batch-size of $20$ for every user and tuned the minimization and maximization stepsize parameters $\eta_1 = \eta_2 = 10^{-4}$ while applying $10$ maximization steps per minimization step. For the $L_2$-regularization penalty, we tuned a coefficient of $\lambda = 4$ for the CIFAR-10 experiments and $\lambda = 1$ for the MNIST experiments. For baseline methods, we used the the following three methods: (1) standard \texttt{FedAvg}  \cite{mcmahan2017communication}, (2) localized \texttt{FedAvg} (\texttt{L-FedAvg})  where each client personalizes the final shared model of \texttt{FedAvg} by locally updating it via $500$ additional local iterations, (3)  federated  model interpolation (\texttt{FedMI}) \cite{mansour2020three} where each client averages the global and its own local models, and { (4) federated first-order model agnostic meta learning (\texttt{Fed-FOMAML}) \cite{fallah2020personalized} applying a first-order meta learning approach to update the local models}. Note that our evaluation metric is the test accuracy averaged over the individual distributions of the $n=100$ nodes. 

Table \ref{tbl:FedOT_AlexNet} includes the test accuracy scores of our experiments with the AlexNet architecture. In these experiments, we applied affine distribution shifts for the MNIST and CIFAR-10 experiments and used color transformation shifts for the Colored-MNIST experiments. As shown by our numerical results, \texttt{FedOT} consistently outperformed the baseline methods in all the experiments and with a definitive margin which was above $15\%$ in six of the eight experimental settings. Similarly, Table \ref{tbl:FedOT_InceptionNet} shows that \texttt{FedOT} also achieves the best performance for the InceptionNet architecture. Overall, our numerical results indicate that \texttt{FedOT} can lead to a significant performance improvement when the learners can learn and reverse the underlying distribution shifts via the optimal transport-based framework.   {
\vspace{-2mm}
\section{Conclusion}
\vspace{-2mm}
In this paper, we introduced the optimal transport-based \texttt{FedOT} framework to address the federated learning problem under heterogeneous data distributions. The \texttt{FedOT} framework leverages multi-input optimal transport costs to measure the discrepancy between the input distributions and also learn the transportation maps needed for transferring the input distributions to a common probability domain. We demonstrated that such a transportation to a common distribution offers an improved generalization and optimization performance in learning the personalized prediction models. In addition, the optimal transport-based analysis results in an upper-bound on the statistical complexity of the federated learning problem. The applied approach can be potentially useful for bounding the sample complexity of learning under heterogeneous data distributions which appear in other transfer and meta learning settings, and can complement information theoretic tools for deriving lower bounds on the statistical complexity. An interesting future direction is to analyze the tightness of the generalization error bound in Section \ref{sec: guarantees} through developing information theoretic lower-bounds on the sample complexity of learning under different input distributions.   
}

{\small
\bibliographystyle{IEEEtran}
\bibliography{main}
}

\clearpage
\newpage

\section{Appendix}
\subsection{Proof of Proposition 1}

Here, we provide a new proof for this result. We use the definition of $n$-ary optimal transport costs with the cost function in the theorem to obtain:
\begin{align*}
    W_c(P_1,\cdots,P_n) &\stackrel{(a)}{=} \min_{\pi\in\Pi(P_1,\cdots,P_n)}\mathbb{E}_{\pi}\biggl[\min_{x'}\:\sum_{i=1}^n \tilde{c}(x',X_i) \biggr] \\
    &\stackrel{(b)}{=} \min_{Q, \pi\in\Pi(P_1,\cdots,P_n,Q)}\mathbb{E}_{X'\sim Q,\pi}\biggl[\sum_{i=1}^n \tilde{c}(X',X_i) \biggr] \\
    &\stackrel{(c)}{=} \min_{Q, \pi\in\Pi(P_1,\cdots,P_n,Q)}\sum_{i=1}^n\mathbb{E}\bigl[ \tilde{c}(X',X_i) \bigr] \\
    &\stackrel{(d)}{=} \min_Q\: \sum_{i=1}^n\min_{\pi\in\Pi(Q,P_i)}\mathbb{E}_{\pi}\bigl[ \tilde{c}(X',X_i) \bigr] \\
    &\stackrel{(e)}{=}\min_Q\: \sum_{i=1}^n W_{\tilde{c}}(Q,P_i).
\end{align*}
In the above, (a) is a direct consequence of the definition of infimal convolution optimal transport costs. We claim that (b) holds because: 1) the solution $X'$ minimizing $\sum_{i=1}^n \tilde{c}(x',X_i)$ is a function of $X_1,\ldots , X_n$ and hence a random variable with a probability distribution which implies
\begin{equation*}
    \min_{\pi\in\Pi(P_1,\cdots,P_n)}\mathbb{E}_{\pi}\biggl[\min_{x'}\:\sum_{i=1}^n \tilde{c}(x',X_i) \biggr] \ge \min_{Q, \pi\in\Pi(P_1,\cdots,P_n,Q)}\mathbb{E}_{X'\sim Q,\pi}\biggl[\sum_{i=1}^n \tilde{c}(X',X_i) \biggr].
\end{equation*}
Moreover, for any random variable $X'$ the following holds almost surely
\begin{equation*}
    \min_{x'}\:\sum_{i=1}^n \tilde{c}(x',X_i) \le \sum_{i=1}^n \tilde{c}(X',X_i),
\end{equation*}
which results in 
\begin{equation*}
    \min_{\pi\in\Pi(P_1,\cdots,P_n)}\mathbb{E}_{\pi}\biggl[\min_{x'}\:\sum_{i=1}^n \tilde{c}(x',X_i) \biggr] \le \min_{Q, \pi\in\Pi(P_1,\cdots,P_n,Q)}\mathbb{E}_{X'\sim Q,\pi}\biggl[\sum_{i=1}^n \tilde{c}(X',X_i) \biggr],
\end{equation*}
that means that (b) is true. (c) is a direct result of the linearity of expectation. Also, as the summation of minimums lower-bounds the minimum of summation:
\begin{align*}
    \min_{Q, \pi\in\Pi(P_1,\cdots,P_n,Q)}\sum_{i=1}^n\mathbb{E}\bigl[ \tilde{c}(X',X_i) \bigr] 
    &\ge \min_{Q}\sum_{i=1}^n \min_{ \pi\in\Pi(P_1,\cdots,P_n,Q)}\mathbb{E}_{\pi}\bigl[ \tilde{c}(X',X_i) \bigr] \\
    & = \min_Q\: \sum_{i=1}^n\min_{\pi\in\Pi(Q,P_i)}\mathbb{E}_{\pi}\bigl[ \tilde{c}(X',X_i) \bigr].
\end{align*}
On the other hand, note that if $Q^*$ together with the conditional distributions $(\pi^*_{X_i|X'})_{i=1}^n$ achieves the minimized value on the right hand side of the above inequality then one can find $\pi\in\Pi(P_1,\cdots,P_n,Q)$ that achieves the same value of $\sum_{i=1}^n\mathbb{E}[ \tilde{c}(X',X_i) ] $. To do this, one draws $(X_1,\cdots,X_n,X')$ by first drawing $X'\sim Q^*$ and then conditionally drawing every $X_i\sim \pi^*_{X_i|X'}(\cdot|x')$.  Therefore, we also have
\begin{align*}
    \min_{Q, \pi\in\Pi(P_1,\cdots,P_n,Q)}\sum_{i=1}^n\mathbb{E}\bigl[ \tilde{c}(X',X_i) \bigr] 
    \le\min_Q\: \sum_{i=1}^n\min_{\pi\in\Pi(Q,P_i)}\mathbb{E}_{\pi}\bigl[ \tilde{c}(X',X_i) \bigr],
\end{align*}
which means (d) holds as well. Finally, (e) is a direct result of the definition of optimal transport costs. As a result, the proof is complete.

\subsection{Proof of Theorem 1}

To prove this result, we apply Proposition \ref{Proposition: n-ary Wasserstein to standard} together with the standard Kantorovich duality theorem \cite{villani2009optimal} to obtain
\begin{align*}
    W_c(P_1,\cdots,P_n) &\stackrel{(a)}{=} \min_Q\: \sum_{i=1}^n W_{\tilde{c}}(Q,P_i) \\
    & \stackrel{(b)}{=} \min_Q\: \sum_{i=1}^n \max_{\phi_i}\:\bigl\{ \mathbb{E}_{P_i}\bigl[\phi^{\tilde{c}}_i(\mathbf{X})\bigr] - \mathbb{E}_Q\bigl[\phi_i(\mathbf{X})\bigr] \bigr\} \\
    & \stackrel{(c)}{=} \min_Q\: \max_{\phi_{1:n}}\: \sum_{i=1}^n\bigl\{ \mathbb{E}_{P_i}\bigl[\phi^{\tilde{c}}_i(\mathbf{X})\bigr] - \mathbb{E}_Q\bigl[\phi_i(\mathbf{X})\bigr] \bigr\} \\ 
     & \stackrel{(d)}{=} \min_Q\: \max_{\phi_{1:n}}\: - \mathbb{E}_Q\bigl[\sum_{i=1}^n\phi_i(\mathbf{X})\bigr] + \sum_{i=1}^n \mathbb{E}_{P_i}\bigl[\phi^{\tilde{c}}_i(\mathbf{X})\bigr]  \\ 
     & \stackrel{(e)}{=} \max_{\phi_{1:n}}\: \min_Q\:  - \mathbb{E}_Q\bigl[\sum_{i=1}^n\phi_i(\mathbf{X})\bigr] + \sum_{i=1}^n \mathbb{E}_{P_i}\bigl[\phi^{\tilde{c}}_i(\mathbf{X})\bigr]  \\
     & \stackrel{(f)}{=} \max_{\phi_{1:n}}\:\biggl\{ -\max_{\mathbf{x}}\bigl\{\sum_{i=1}^n\phi_i(\mathbf{x})\bigr\} + \sum_{i=1}^n \mathbb{E}_{P_i}\bigl[\phi^{\tilde{c}}_i(\mathbf{X})\bigr] \biggr\} \\
     & \stackrel{(g)}{=} \max_{\phi_{1:n}\atop \forall \mathbf{x}: \sum_i\phi_i(\mathbf{x})\le 0 }\: \sum_{i=1}^n \mathbb{E}_{P_i}\bigl[\phi^{\tilde{c}}_i(\mathbf{X})\bigr] \\ 
     & \stackrel{(h)}{=} \max_{\phi_{1:n}\atop \forall \mathbf{x}: \sum_i\phi_i(\mathbf{x})= 0 }\: \sum_{i=1}^n \mathbb{E}_{P_i}\bigl[\phi^{\tilde{c}}_i(\mathbf{X})\bigr].
\end{align*}

Here, (a) rewrites the result of Proposition \ref{Proposition: n-ary Wasserstein to standard}. (b) follows from the standard Kantorovich duality theorem \cite[Theorem 5.10]{villani2009optimal}. (c) uses the fact that the maximization problems inside the summation optimize independent function variables. (d) is a consequence of the linearity of expectations. (e) applies a modified minimax theorem \cite{pratelli2005minimax} which holds under the assumptions that the objective is convex in $Q$, concave in $\phi_{1:n}$ functional variables, and continuous in both minimization and maximization variables ($\tilde{c}$ is assumed to be a continuous cost), the feasible sets are convex and the minimization feasible set is compact. 

(f) follows from minimizing the summation of $\phi_{1:n}$. (g) holds because the objective is invariant to adding a constant to any $\phi_i$'s and hence one can define an auxiliary optimization variable $t=\max_{\mathbf{x}}\sum_i \phi_i(\mathbf{x})$ and constrain it to be upper-bounded by zero. Finally, (h) holds because the c-transform operation is monotonically increasing in the output of $\phi_{1:n}$. Therefore, the proof is complete. 

\subsection{Proof of Theorem 2}

We reverse the proof of Theorem \ref{Thm: Kantorovich duality} to show that the optimal $\phi^*_{1:n}$ also provide a solution to the following optimization problems:
\begin{align*}
     &\max_{\substack{\phi_{1:n}:\, \text{\rm convex}\\
    \forall \mathbf{x}: \, \frac{1}{n}\sum_i \phi_i(\mathbf{x})\le\frac{1}{2}\Vert \mathbf{x}\Vert_2^2
    }}\;
    \sum_{i=1}^n \mathbb{E}_{P_i}\bigl[\frac{1}{2}\Vert \mathbf{X}\Vert^2 - \phi_i^\star(\mathbf{X})\bigr] \\
    =\, & \max_{\substack{\phi_{1:n}:\, \text{\rm convex}
    }}\: \min_Q\: 
    \sum_{i=1}^n\biggl[\mathbb{E}_{Q}\bigl[\frac{1}{2}\Vert \mathbf{X}\Vert^2 - \phi_i(\mathbf{X})\bigr] +  \mathbb{E}_{P_i}\bigl[\frac{1}{2}\Vert \mathbf{X}\Vert^2 - \phi_i^\star(\mathbf{X})\bigr]\biggr] \\
    =\, & \min_Q \: \max_{\substack{\phi_{1:n}:\, \text{\rm convex}
    }}\:  
    \sum_{i=1}^n\biggl[\mathbb{E}_{Q}\bigl[\frac{1}{2}\Vert \mathbf{X}\Vert^2 - \phi_i(\mathbf{X})\bigr] +  \mathbb{E}_{P_i}\bigl[\frac{1}{2}\Vert \mathbf{X}\Vert^2 - \phi_i^\star(\mathbf{X})\bigr]\biggr] \\
    =\, & \min_Q \: \:  
    \sum_{i=1}^n\max_{\substack{\phi_{i}:\, \text{\rm convex}
    }}\biggl\{\mathbb{E}_{Q}\bigl[\frac{1}{2}\Vert \mathbf{X}\Vert^2 - \phi_i(\mathbf{X})\bigr] +  \mathbb{E}_{P_i}\bigl[\frac{1}{2}\Vert \mathbf{X}\Vert^2 - \phi_i^\star(\mathbf{X})\bigr]\biggr\}.
    \end{align*}
Note that for any optimal $\phi^*_{1:n}$ to the $n$-ary 2-Wasserstein dual problem \eqref{Eq: W2_nary_dual}, there will exist a distribution $Q^*$ that together with $\phi^*_{1:n}$ solve the above min-max problem. Therefore, as the inner maximization problem inside the summation represents the dual problem to the standard binary 2-Wasserstein cost, Brenier's theorem \cite{villani2009optimal} suggests that for every $\mathbf{X}_i\sim P_i$, $\nabla {\phi^{*^\star}_i} (\mathbf{X}_i)$ is distributed according to $Q^*$. Therefore, the following holds and the proof is complete:
\begin{equation}
    \forall\, 1\le i,j\le n:\quad \nabla {\phi^{*^\star}_i} (\mathbf{X}_i) \stackrel{\text{\rm dist}}{=} \nabla {\phi^{*^\star}_j} (\mathbf{X}_j). \nonumber
\end{equation}

\subsection{Proof of Proposition 2}

To show this result, we start by proving the following lemma.
\begin{lemma}\label{Lemma: Prop 2}
Consider a $\gamma$-smooth function $\phi$ with $\gamma<1$, i.e., $\nabla \phi$ is a $\gamma$-Lipschitz function. Then, 
\begin{equation*}
    \phi^{\tilde{c}_2}(\mathbf{x})\ge \phi(\mathbf{x}) - \frac{1}{1-\gamma} \Vert \nabla \phi(\mathbf{x})\Vert^2_2.
\end{equation*}
\end{lemma}
\begin{proof}
According to the definition of $\gamma$-smooth functions, we have
\begin{equation*}
   \forall \mathbf{x},\mathbf{x}':\quad \phi(\mathbf{x}') \ge \phi(\mathbf{x}) + \nabla \phi(\mathbf{x})^{\top}(\mathbf{x}'-\mathbf{x}) - \frac{\gamma}{2}\Vert \mathbf{x}'-\mathbf{x}\Vert^2_2.
   \end{equation*}
 Plugging the above inequality into the definition of c-transform shows that
 \begin{align*}
      \phi^{\tilde{c}_2}(\mathbf{x}) &= \min_{\mathbf{x}'}\:\bigl\{ \phi(\mathbf{x}') + \frac{1}{2}\Vert \mathbf{x}'-\mathbf{x}\Vert^2_2\bigr\}  \\
      &\ge \min_{\mathbf{x}'}\:\bigl\{ \phi(\mathbf{x}) + \nabla\phi(\mathbf{x})^{\top}(\mathbf{x}'-\mathbf{x}) + \frac{1-\gamma}{2}\Vert \mathbf{x}'-\mathbf{x}\Vert^2_2\bigr\}  \\
      &= \phi(\mathbf{x}) + \min_{\mathbf{x}'}\:\bigl\{\nabla \phi(\mathbf{x})^{\top}(\mathbf{x}'-\mathbf{x}) + \frac{1-\gamma}{2}\Vert \mathbf{x}'-\mathbf{x}\Vert^2_2\bigr\} \\
      &= \phi(\mathbf{x}) + \min_{\mathbf{x}'}\:\bigl\{ \nabla\phi(\mathbf{x})^{\top}\mathbf{x}'+ \frac{1-\gamma}{2}\Vert \mathbf{x}'\Vert^2_2\bigr\} \\
      &= \phi(\mathbf{x}) -  \frac{1}{2(1-\gamma)}\Vert \nabla\phi(\mathbf{x})\Vert^2_2.
 \end{align*}
 Therefore, the lemma's proof is complete.
\end{proof}
Based on Lemma \ref{Lemma: Prop 2}, we only need to show that under the proposition's assumptions we have:
\begin{equation*}
    \big\Vert \nabla(\mathbf{v}_i\phi_{\mathbf{U}})(\mathbf{x})\big\Vert^2_2 \le \Vert \mathbf{v}_i\Vert^2_2 + \Vert\mathbf{U}\Vert^2_F.
\end{equation*}
However, since the neural network's activation function is assumed to be $1$-Lipschitz we have
\begin{equation}
    \big\Vert \nabla(\mathbf{v}_i\phi_{\mathbf{U}})(\mathbf{x})\big\Vert^2_2 \le \Vert\mathbf{v}_i \Vert^2_2\prod_{i=1}^L \Vert U_i \Vert^2_{2}
\end{equation}
which under the assumptions is upper-bounded by $\Vert\mathbf{v}_i \Vert^2_2$ and hence the proposition's proof is complete. 

\subsection{Proof of Theorem 3}
To show this result for the min-max problem \eqref{Eq: 1-FedOT}, note that if $\epsilon' = \epsilon / (L_{\phi}+1)$, then an $\epsilon'$-covering of $\Phi$ and $\Theta$ will result in an $\epsilon$-covering for $\Phi\circ \Theta :=\{\phi(\psi_{\boldsymbol{\theta}}(\cdot)):\: \phi\in\Phi,\, \boldsymbol{\theta}\in\Theta \}$. This property holds, because under the conditions that $\Vert \phi_1 - \phi_2 \Vert_{\infty}\le \epsilon'$ and $\Vert \theta_1 - \theta_2 \Vert_{\infty}\le \epsilon'$ we have for every $\mathbf{x}$:
\begin{align*}
   \Vert \phi_1(\theta_1(\mathbf{x})) - \phi_2(\theta_2(\mathbf{x})) \Vert_{\infty} &\le \Vert \phi_1(\theta_1(\mathbf{x})) - \phi_1(\theta_2(\mathbf{x})) \Vert_{\infty} + \Vert \phi_1(\theta_2(\mathbf{x})) - \phi_2(\theta_2(\mathbf{x})) \Vert_{\infty} \\
   &\le L_{\phi}\Vert \theta_1(\mathbf{x}) - \theta_2(\mathbf{x}) \Vert_{\infty} + \epsilon' \\
   &\le (L_{\phi} + 1) \epsilon'\\
   & = \epsilon.
\end{align*}
Therefore, $\mathcal{N}(\Phi\circ\Theta ,\epsilon,\Vert\cdot\Vert_{\infty}) \le  \mathcal{N}(\Phi ,\epsilon',\Vert\cdot\Vert_{\infty}) \mathcal{N}(\Theta ,\epsilon',\Vert\cdot\Vert_{\infty})$ and hence
\begin{equation*}
    \log \mathcal{N}(\Phi\circ\Theta ,\epsilon,\Vert\cdot\Vert_{\infty}) \le \log \mathcal{N}(\Phi ,\epsilon',\Vert\cdot\Vert_{\infty}) \, + \log \mathcal{N}(\Theta ,\epsilon',\Vert\cdot\Vert_{\infty}).
\end{equation*}
As a result, since $\sqrt{a+b}\le \sqrt{a}+\sqrt{b}$ holds for every non-negative $a,b\ge 0$: 
\begin{equation}\label{Eq: Proof Theorem 3_1}
    \mathcal{V}(\Phi\circ\Theta) \le \bigl(L_{\phi}+1\bigr)\bigl(\mathcal{V}(\Phi)+\mathcal{V}(\Theta)\bigr).
\end{equation}
Therefore, noting that the loss function is $L_{\ell}$-Lipschitz we can combine standard generalization bounds via Rademacher complexity \cite{bartlett2002rademacher} and the Dudley entropy integral bound \cite{bartlett2017spectrally} to show that for any $\delta>0$ and $1\le i\le n$ with probability at least $1-\delta/2n$ the following holds for every $\phi\in\Phi$ and $\boldsymbol{\theta}\in\Theta$:
\begin{align*}
\bigg\vert\frac{1}{m}\sum_{j=1}^m \lambda \phi^{\tilde{c}_1}(\psi_{\boldsymbol{\theta}}(\mathbf{x}_{i,j})) - \mathbb{E}_{P_i}\bigl[\lambda\phi^{\tilde{c}_1}(\psi_{\boldsymbol{\theta}}(\mathbf{X}))\bigr]\bigg\vert \le \mathcal{O}\left( M L_\phi L_\theta\sqrt{\frac{ \bigl(\mathcal{V}(\Phi)+\mathcal{V}(\Theta)\bigr)^2 \log(n/\delta)}{m}}\right).
\end{align*}
Consequently, applying the union bound indicates that with probability at least $1-\delta / 2$ the following will hold for every $\phi_{1:n}\in\Phi$ and $\boldsymbol{\theta}_{1:n}\in\Theta$
\begin{align*}
    \bigg\vert\frac{1}{mn}\sum_{i=1}^n\sum_{j=1}^m \lambda \phi^{\tilde{c}_1}_i(\psi_{\boldsymbol{\theta}_i}(\mathbf{x}_{i,j})) - \frac{1}{n}\sum_{i=1}^n\mathbb{E}_{P_i}\bigl[\lambda\phi^{\tilde{c}_1}_i(\psi_{\boldsymbol{\theta}_i}(\mathbf{X}))\bigr]\bigg\vert 
    \le \mathcal{O}\left( ML_\phi L_\theta\sqrt{\frac{ \bigl(\mathcal{V}(\Phi)+\mathcal{V}(\Theta)\bigr)^2 \log(n/\delta)}{m}}\right).
\end{align*}
Moreover, note that we assume that for some choice of $\boldsymbol{\theta}_{1:n}$ the $n$-ary 1-Wasserstein distance will be zero and hence for every $\boldsymbol{\theta}_{1:n}$  that can be an optimal solution minimizing the objective function:
\begin{equation}
    W_{c_1}\bigl(P_{\psi_{\boldsymbol{\theta}_1}(\mathbf{X}_1)},\cdots P_{\psi_{\boldsymbol{\theta}_n}(\mathbf{X}_n)}\bigr) \le \frac{M}{\lambda}.
\end{equation}
Therefore, according to Proposition \ref{Proposition: n-ary Wasserstein to standard} we have
\begin{equation}
    \forall 1\le i,j\le n:\;\; W_{1}\bigl(P_{\psi_{\boldsymbol{\theta}_i}(\mathbf{X}_i)}, P_{\psi_{\boldsymbol{\theta}_j}(\mathbf{X}_j)}\bigr) \le \frac{M}{\lambda}.
\end{equation}
Therefore, for every $i>1$ the transferred sample $(\psi_{\boldsymbol{\theta}_i}(\mathbf{X}_{i}),Y_{i})$ has a distribution that has at most $\frac{M}{\lambda}$ 1-Wasserstein distance from $P_{\psi_{\boldsymbol{\theta}_1}(\mathbf{X}),Y}$. As a result, for every $\boldsymbol{\theta}_1\in\Theta$ and $\delta>0$ with probability at least $1-\delta$ the following bound holds for every $\mathbf{w}\in\mathcal{W}$ and minimizing solutions $\boldsymbol{\theta}_{1:n}$:   
\begin{align*}
   \biggl\vert\frac{1}{mn}\sum_{i=1}^n\sum_{j=1}^m  \ell\bigl(f_{\mathbf{w}}(\psi_{\boldsymbol{\theta}_i}(\mathbf{x}_{i,j})),y_{i,j} )\bigr) - \mathbb{E}_{P_{\psi_{\boldsymbol{\theta}_1}(\mathbf{X}),Y}}\bigl[ \ell\bigl(f_{\mathbf{w}}(\mathbf{X}'),Y )\bigr)\bigr] \biggr\vert
   \le \mathcal{O}\left( ML_w\sqrt{\frac{ \mathcal{V}(\mathcal{W})^2 \log(1/\delta)}{mn}} + \frac{M L_{\ell}L_w}{\lambda}\right). 
\end{align*}
Applying the Dudley's entropy theorem by covering $\boldsymbol{\theta}_1\in\Theta$ therefore shows that for every $\delta$ with probability at least $1-\delta/2$ the following holds for every $\mathbf{w}\in\mathcal{W}$ and minimizing solution $\boldsymbol{\theta}_{1:n}$  
\begin{align*}
   &\biggl\vert\frac{1}{mn}\sum_{i=1}^n\sum_{j=1}^m  \ell\bigl(f_{\mathbf{w}}(\psi_{\boldsymbol{\theta}_i}(\mathbf{x}_{i,j})),y_{i,j} )\bigr) - \frac{1}{n}\sum_{i=1}^n\mathbb{E}_{P_i}\bigl[ \ell\bigl(f_{\mathbf{w}}(\psi_{\boldsymbol{\theta}_i}(\mathbf{X})),Y )\bigr)\bigr] \biggr\vert \\
   \le\, & \mathcal{O}\left( ML_w\sqrt{\frac{ \bigl(\mathcal{V}(\mathcal{W})+\mathcal{V}(\Theta)\bigr)^2 \log(1/\delta)}{mn}} + \frac{M L_{\ell}L_w}{\lambda}\right). 
\end{align*}
Combining the above results show that for every $\delta>0$ the following holds with probability at least $1-\delta$ for every $\mathbf{w}\in\mathcal{W}$ and minimizing solution $\boldsymbol{\theta}_{1:n}$  
\begin{align*}
    &\biggl\vert \max_{\substack{\phi_{1:n}\in  \Phi:\atop
   \forall \mathbf{x}:\;\sum_i \phi_i(\mathbf{x})=0
   }}
   \mathcal{L}(\mathbf{w},\boldsymbol{\theta}_{1:n},\phi_{1:n}) 
   -
   \max_{\substack{\phi_{1:n}\in  \Phi:\atop
   \forall \mathbf{x}:\;\sum_i \phi_i(\mathbf{x})=0
   }}
   \widehat{\mathcal{L}}(\mathbf{w},\boldsymbol{\theta}_{1:n},\phi_{1:n}) \biggr\vert \\
    \le\: &
    \mathcal{O}\Biggl({L_{\ell} L_{\bbw} M}\sqrt{\frac{\bigl(\mathcal{V}(\mathcal{W})+\mathcal{V}(\Theta)\big)^2\log(1/\delta)}{mn}} +{\lambda L_\phi L_\theta M} \sqrt{\frac{\big(\mathcal{V}(\Phi)+\mathcal{V}(\Theta)\big)^2\log(n/\delta)}{m}}+\frac{ML_w L_\ell}{\lambda} \Biggr),
\end{align*}
which shows that the theorem's result holds as well, because $|\min_{\theta\in\Theta} f_1(\theta) - \min_{\theta\in\Theta} f_2(\theta)|\le \max_{\theta\in\Theta} |f_1(\theta) -f_2(\theta)|$ holds for any functions $f_1,f_2$ and feasible set $\Theta$. The theorem's proof is therefore complete.

We, furthermore, note that by optimizing $\lambda$ in the upper-bound we can show the following bound for the optimal value $\lambda^* $
\begin{align*}
    &\biggl\vert \min_{\boldsymbol{\theta}_{1:n}}\max_{\substack{\phi_{1:n}\in  \Phi:\atop
   \forall \mathbf{x}:\;\sum_i \phi_i(\mathbf{x})=0
   }}
   \mathcal{L}(\mathbf{w},\boldsymbol{\theta}_{1:n},\phi_{1:n}) 
   -
   \min_{\boldsymbol{\theta}_{1:n}}\max_{\substack{\phi_{1:n}\in  \Phi:\atop
   \forall \mathbf{x}:\;\sum_i \phi_i(\mathbf{x})=0
   }}
   \widehat{\mathcal{L}}(\mathbf{w},\boldsymbol{\theta}_{1:n},\phi_{1:n}) \biggr\vert \\
    \le\: &
    \mathcal{O}\Biggl({L_{\ell} L_{\bbw} M}\sqrt{\frac{\bigl(\mathcal{V}(\mathcal{W})+\mathcal{V}(\Theta)\big)^2\log(1/\delta)}{mn}} +M \sqrt[4]{\frac{\bigl(L_\phi L_\theta L_w L_\ell\big(\mathcal{V}(\Phi)+\mathcal{V}(\Theta)\big)\bigr)^2\log(n/\delta)}{m}} \Biggr),
\end{align*}
Note that for the above bound to hold the value of $\lambda^*$ will be determined as:
\begin{equation*}
    \lambda^* = \sqrt{\frac{L_\ell L_w m^{1/2}}{L_\phi L_\theta\big(\mathcal{V}(\Phi)+\mathcal{V}(\Theta)\big)\log(n/\delta)^{1/2}}}.
\end{equation*}

    
\subsection{Proof of Theorem \ref{thm: opt}}

In this section, we provide the detailed proof of Theorem \ref{thm: opt} by first laying out some useful lemmas which we directly use in our proof. Some of these lemmas follow similar steps as in \cite{reisizadeh2020robust}. However, to be self-contained, we provide complete proofs of all the lemmas. Let us first set up our notations which are summarized in Table \ref{Table: Notations}.

\begin{table}[h!] 
\begin{center}
    \begin{tabular}{ c c }
    \toprule
    Notation & Description  \\
    \midrule
    $\displaystyle \wbar_t = \frac{1}{n} \sum_{i \in [n]} \bbw^i_t$ 
    & average model $\bbw$ at iteration $t$ \vspace{0cm}\\
    $\displaystyle \vbar_t = \frac{1}{n} \sum_{i \in [n]} \bbv^i_t$ 
    & average model $\bbv$ at iteration $t$ \vspace{0cm}\\
    $\displaystyle b_t = \E [ \Lambda(\wbar_t) - \Lhat(\wbar_t, \vbar_t)]$ & \begin{tabular}{@{}c@{}}optimality gap measure \\ between $\Lhat(\wbar_t, \vbar_t)$ and $\max_{\bbv} \Lhat(\wbar_t, \bbv)$ \end{tabular} \vspace{0.2cm}\\
    $\displaystyle e_t = \frac{1}{n} \sum_{i \in [n]} \E \norm{\w^i_t - \wbar_{t}}^2 $ & \begin{tabular}{@{}c@{}}average deviation of the local models $\w^i_t$  \\ from the average model at iteration $t$ \end{tabular} \\
    $\displaystyle E_t = \frac{1}{n} \sum_{i \in [n]} \E \norm{\bbv^i_t - \vbar_{t}}^2 $ & \begin{tabular}{@{}c@{}}average deviation of the local models $\bbv^i_t$ \\ from the average model at iteration $t$ \end{tabular} \\
    $\displaystyle g_t 
    = \E \norm{\frac{1}{n} \sum_{i \in [n]} \grw \Lhat_i(\w^i_t , \bbv^i_t)}^2$  & \begin{tabular}{@{}c@{}}norm squared of  \\ local gradients w.r.t $\w$ at iteration $t$ \end{tabular} \\
    $\displaystyle G_t 
    = \E \norm{\frac{1}{n} \sum_{i \in [n]} \grv \Lhat_i(\w^i_t , \bbv^i_t)}^2$  & \begin{tabular}{@{}c@{}}norm squared of  \\ local gradients w.r.t $\bbv$ at iteration $t$ \end{tabular} \\
    $\displaystyle h_t 
    = \E \norm{\gr \Lambda(\wbar_{t}) - \frac{1}{n} \sum_{i \in [n]} \grw \Lhat_i(\w^i_t , \bbv^i_t)}^2$ & \begin{tabular}{@{}c@{}}norm squared of deviation in gradients w.r.t $\w$ \\ of $\max_{\bbv} \Lhat(\wbar_t, \bbv)$ and local functions $\Lhat_i(\w^i_t , \bbv^i_t)$ \end{tabular}\\
    \bottomrule
    \end{tabular}
    \vspace{1mm}
    \caption{Table of notations.}
    \label{Table: Notations}
\end{center}
\end{table}

Let us state the following assumption used in the statement of Theorem \ref{thm: opt}.

\begin{assumption} \label{assumption: stoch gr}
Local stochastic gradients are unbiased and variance bounded, i.e., 
\begin{align}
    \E \Vert \sgr_{\bbw}\ccalLhat_i (\bbw, \bbv) - \grw \ccalLhat_i(\bbw, \bbv) \Vert^2 \leq \sigmaw, \\
    \E \Vert \sgrv\ccalLhat_i (\bbw, \bbv) - \gr_{\bbv} \ccalLhat_i(\bbw, \bbv) \Vert^2 \leq \sigmav.
\end{align}
\end{assumption}

Now we lay out some useful and preliminary lemmas which we later employ in the main proof.

\subsection{Useful lemmas}

\begin{lemma}[\cite{lin2019gradient}] \label{lemma:L_Lambda}
If Assumption \ref{assumption: loss} (ii) holds, i.e., each of the local losses $\Lhat_i$ have $L$-Lipschitz gradients and  $\Lhat_i(\cdot, \bbv)$ are $\mu$-strongly concave, then 
\begin{align}
    \gr \Lambda(\bbw)
    =
    \grw \Lhat(\bbw, \bbv^*(\bbw)),
\end{align}
where $\bbv^*(\bbw) \in \argmax_{\bbv} \Lhat(\bbw, \bbv)$ for any $\bbw$. Moreover, $\Lambda$ has Lipschitz gradients with parameter $L_{\Lambda} = (\kappa + 1) L$.
\end{lemma}

\begin{lemma} \label{Lemma: Lambda contraction}
If Assumptions \ref{assumption: loss} and \ref{assumption: stoch gr} hold, then the iterates of {\normalfont \texttt{FedOT-GDA}} satisfy the following contraction inequality for any iteration $t \geq 0$
\begin{align}
    \E [ \Lambda(\wbar_{t+1}) ] - \E [ \Lambda(\wbar_{t}) ]
    \leq
    - \frac{\eta_1}{2} \E \norm{\gr \Lambda(\wbar_{t})}^2 
    +
    \frac{\eta_1}{2} h_t
    -
    \frac{\eta_1}{2} \left( 1 - \eta_1 L_{\Lambda} \right) g_t
    +
    \eta_1^2 \frac{L_{\Lambda}}{2} \frac{\sigmaw}{n}.
\end{align}
\end{lemma}

\begin{lemma} \label{Lemma: h_t bound}
If Assumption \ref{assumption: loss} (ii) holds, then we have
\begin{align} \label{eq:h_t-final}
    h_t
    \leq
    \frac{4 L^2}{\mu} b_t 
    +
    2 L^2 e_t
    +
    2 L^2 E_t.
\end{align}
\end{lemma}

\begin{lemma}\label{lemma: sum e_t + E_t}
If Assumptions \ref{assumption: loss} and \ref{assumption: stoch gr} hold and the step-sizes $\eta_1, \eta_2$ satisfy $32 (\tau - 1)^2 L^2 (\eta_1^2 + \eta_2^2) \leq 1$, then the average of the sequence $e_t$ and $E_t$ over $t=0,\cdots,T-1$ is bounded as follows
\begin{align}
    \frac{1}{T} \sum_{t=0}^{T-1} e_{t}
    +
    \frac{1}{T} \sum_{t=0}^{T-1} E_{t}
    &\leq
    20 \eta_1^2 (\tau - 1)^2 \frac{1}{T} \sum_{t=0}^{T-1} g_t
    +
    20 \eta_2^2 (\tau - 1)^2 \frac{1}{T} \sum_{t=0}^{T-1} G_t \\
    &\quad
    +
    16 \eta_1^2 (\tau - 1)^2 \rhow^2
    +
    16 \eta_2^2 (\tau - 1)^2 \rhov^2\\
    &\quad
    +
    4 \eta_1^2 (\tau - 1) (n+1) \frac{\sigmaw}{n}
    +
    4 \eta_2^2 (\tau - 1) (n+1) \frac{\sigmav}{n}.
\end{align}
\end{lemma}

\begin{lemma} \label{lemma: sum b_t}
If Assumptions \ref{assumption: loss} and \ref{assumption: stoch gr} hold and the step-sizes step-sizes $\eta_1, \eta_2$ satisfy $\frac{\eta_1}{\eta_2} \leq \frac{1}{8 \kappa^2}$, then the average of the sequence  $\{b_t\}_{t = 0}^{T-1}$ scan be bounded as follows
\begin{align} 
    \frac{1}{T} \sum_{t=0}^{T-1} b_{t}
    &\leq
    \frac{L^2}{\mu^2} \frac{D^2}{\eta_2 T}
    +
    \frac{\eta_1}{\eta_2} \frac{1}{\mu} \frac{1}{T} \sum_{t=0}^{T-1} \E \norm{\gr \Lambda(\wbar_{t})}^2 \\
    &\quad +
    \frac{\eta_1^2}{\eta_2} \frac{1}{\mu_2 n} \left( L + L_{\Lambda} + 2 \eta_2 L^2 \right) \frac{1}{T} \sum_{t=0}^{T-1} g_{t}
    -
    \frac{1}{\mu} (1 - \eta_2 L) \frac{1}{T} \sum_{t=0}^{T-1} G_{t}\\
    &\quad +
    +
    \frac{\eta_1 + \eta_2}{\eta_2} \frac{2L^2}{\mu} \left( \frac{1}{T} \sum_{t=0}^{T-1} e_{t} + \frac{1}{T} \sum_{t=0}^{T-1} E_{t}\right) \\
    &\quad +
    \frac{\eta_1^2}{\eta_2} \frac{1}{\mu} \left(  L + L_{\Lambda} + 2 \eta_2 L^2 \right) \frac{\sigmaw}{n} 
    +
    \eta_2 \frac{L}{\mu} \frac{\sigmav}{n}.
\end{align}
\end{lemma}

\begin{lemma}\label{lemma: sum h_t}
If Assumptions \ref{assumption: loss} and \ref{assumption: stoch gr} hold and the step-sizes $\eta_1, \eta_2$ satisfy $\frac{\eta_1}{\eta_2} \leq \frac{1}{8 \kappa^2}$ and $32 (\tau - 1)^2 L^2 (\eta_1^2 + \eta_2^2) \leq 1$, then the average of the sequence $h_t$ over $t=0,\cdots,T-1$ is bounded as follows
\begin{align}
    \frac{1}{T} \sum_{t=0}^{T-1} h_{t}
    &\leq
    \frac{4 L^4}{\mu^3} \frac{D^2}{\eta_2 T}
    +
    \frac{\eta_1}{\eta_2} \frac{4 L^2}{\mu^2} \frac{1}{T} \sum_{t=0}^{T-1} \E \norm{\gr \Lambda(\wbar_{t})}^2 \\
    &\quad +
    \frac{4 L^2}{\mu^2} \frac{\eta_1^2}{\eta_2} \left( L + L_{\Lambda} + 2 \eta_2 L^2 \right) \frac{1}{T} \sum_{t=0}^{T-1} g_{t}
    -
    \frac{4 L^2}{\mu^2} \left( 1 - \eta_2 L \right) \frac{1}{T} \sum_{t=0}^{T-1} G_{t}\\
    &\quad
    +
    \left( 2L^2 + \frac{8L^4}{\mu^2}(\frac{\eta_1}{\eta_2} + 1) \right) \frac{1}{T} \sum_{t=0}^{T-1} (e_{t} + E_{t}) \\
    &\quad +
    \frac{\eta_1^2}{\eta_2} \frac{4 L^2}{\mu^2} \left( L + L_{\Lambda} + 2 \eta_2 L^2 \right) \frac{\sigmaw}{n} 
    +
    \eta_2 \frac{4 L^3}{\mu^2} \frac{\sigmav}{n} .
\end{align}
\end{lemma}

Having set the main preliminary lemmas, we proceed to prove Theorem \ref{thm: opt}.


\subsection{Proof of Theorem \ref{thm: opt}}

Using Lemma \ref{Lemma: Lambda contraction}, we can write
\begin{align}
    \frac{1}{T} \left( \E [ \Lambda(\wbar_{T}) ] - \Lambda(\wbar_0) \right)
    &\leq
    - 
    \frac{\eta_1}{2} \frac{1}{T} \sum_{t=0}^{T-1} \E \norm{\gr \Lambda(\wbar_{t})}^2 
    +
    \frac{\eta_1}{2} \frac{1}{T} \sum_{t=0}^{T-1} h_t
    -
    \frac{\eta_1}{2} \left(1 - \eta_1 L_{\Lambda} \right) \frac{1}{T} \sum_{t=0}^{T-1} g_{t}
    +
    \eta_1^2 \frac{L_{\Lambda}}{2} \frac{\sigmaw}{n}.
\end{align}
Next, we substitute $\frac{1}{T} \sum_{t=0}^{T-1} h_t$ from Lemma \ref{lemma: sum h_t}, which yields
\begin{align}
    \frac{1}{T} \left( \E [ \Lambda(\wbar_{T}) ] - \Lambda(\wbar_0) \right)
    &\leq
    - 
    \frac{\eta_1}{2} \frac{1}{T} \sum_{t=0}^{T-1} \E \norm{\gr \Lambda(\wbar_{t})}^2 \\
    &\quad 
    +
    \frac{\eta_1}{2} \frac{4 L^4}{\mu^3} \frac{D^2}{\eta_2 T}\\
    &\quad 
    +
    \frac{\eta_1}{2} \frac{\eta_1}{\eta_2} \frac{4 L^2}{\mu^2} \frac{1}{T} \sum_{t=0}^{T-1} \E \norm{\gr \Lambda(\wbar_{t})}^2\\
    &\quad 
    +
    \frac{\eta_1}{2} \frac{4 L^2}{\mu^2} \frac{\eta_1^2}{\eta_2} \left( L + L_{\Lambda} + 2 \eta_2 L^2 \right) \frac{1}{T} \sum_{t=0}^{T-1} g_{t}\\
    &\quad 
    -
    \frac{\eta_1}{2} \frac{4 L^2}{\mu^2} \left( 1 - \eta_2 L \right) \frac{1}{T} \sum_{t=0}^{T-1} G_{t}\\
    &\quad 
    +
    \frac{\eta_1}{2} \left( 2L^2 + \frac{8L^4}{\mu^2}(\frac{\eta_1}{\eta_2} + 1) \right) \frac{1}{T} \sum_{t=0}^{T-1} (e_{t} + E_{t})\\
    &\quad 
    +
    \frac{\eta_1}{2} \frac{\eta_1^2}{\eta_2} \frac{4 L^2}{\mu^2} \left( L + L_{\Lambda} + 2 \eta_2 L^2 \right) \frac{\sigmaw}{n} 
    +
    \frac{\eta_1}{2} \eta_2 \frac{4 L^3}{\mu^2} \frac{\sigmav}{n}\\
    &\quad 
    -
    \frac{\eta_1}{2} \left(1 - \eta_1 L_{\Lambda} \right) \frac{1}{T} \sum_{t=0}^{T-1} g_{t}\\
    &\quad +
    \eta_1^2 \frac{L_{\Lambda}}{2} \frac{\sigmaw}{n}.
\end{align}
After regrouping the terms and adopting the notation $\hat{L} \coloneqq L + L_{\Lambda} + 2 \eta_2 L^2$, we can further write
\begin{align}
    \frac{1}{T} \left( \E [ \Lambda(\wbar_{T}) ] - \Lambda(\wbar_0) \right)
    &\leq
    - 
    \frac{\eta_1}{2} \left( 1 - \frac{\eta_1}{\eta_2} \frac{4 L^2}{\mu^2} \right) \frac{1}{T} \sum_{t=0}^{T-1} \E \norm{\gr \Lambda(\wbar_{t})}^2 \\
    &\quad 
    +
    \frac{\eta_1}{\eta_2 T} \frac{2 L^4}{\mu^3} D^2\\
    &\quad 
    - \left( \frac{\eta_1}{2} \left(1 - \eta_1 L_{\Lambda} \right) - \frac{\eta_1^3}{\eta_2} \frac{2 L^2}{\mu^2}  \hat{L} \right)  \frac{1}{T} \sum_{t=0}^{T-1} g_{t} \\
    &\quad 
    -
    \frac{\eta_1}{2} \frac{4 L^2}{\mu^2} \left( 1 - \eta_2 L \right) \frac{1}{T} \sum_{t=0}^{T-1} G_{t} \\
    &\quad 
    +
    \frac{\eta_1}{2} \left( 2L^2 + \frac{8L^4}{\mu^2}(\frac{\eta_1}{\eta_2} + 1) \right) \frac{1}{T} \sum_{t=0}^{T-1} (e_{t} + E_{t}) \\
    &\quad 
    +
    \frac{\eta_1^3}{\eta_2} \frac{2 L^2}{\mu^2}  \hat{L} \frac{\sigmaw}{n} 
    +
    \eta_1^2 \frac{L_{\Lambda}}{2} \frac{\sigmaw}{n}
    +
    \frac{\eta_1}{2} \eta_2 \frac{4 L^3}{\mu^2} \frac{\sigmav}{n}.
\end{align}
Next, we substitute $\frac{1}{T} \sum_{t=0}^{T-1} (e_{t} + E_{t})$ from Lemma \ref{lemma: sum e_t + E_t}, which implies that if the step-sizes satisfy $32 (\tau - 1)^2 L^2 (\eta_1^2 + \eta_2^2) \leq 1$, then
\begin{align}
    \frac{1}{T} \left( \E [ \Lambda(\wbar_{T}) ] - \Lambda(\wbar_0) \right)
    &\leq
    - 
    \frac{\eta_1}{2} \left( 1 - \frac{\eta_1}{\eta_2} \frac{4 L^2}{\mu^2} \right) \frac{1}{T} \sum_{t=0}^{T-1} \E \norm{\gr \Lambda(\wbar_{t})}^2 \\
    &\quad 
    +
    \frac{\eta_1}{\eta_2 T} \frac{2 L^4}{\mu^3} D^2\\
    &\quad 
    +
    \frac{\eta_1^3}{\eta_2} \frac{2 L^2}{\mu^2}  \hat{L} \frac{\sigmaw}{n} 
    +
    \eta_1^2 \frac{L_{\Lambda}}{2} \frac{\sigmaw}{n}
    +
    \frac{\eta_1}{2} \eta_2 \frac{4 L^3}{\mu^2} \frac{\sigmav}{n}\\
    &\quad 
    +
    \frac{\eta_1}{2} \left( 2L^2 + \frac{8L^4}{\mu^2}(\frac{\eta_1}{\eta_2} + 1) \right) 4 \eta_1^2 (\tau - 1) (n+1) \frac{\sigmaw}{n}\\
    &\quad
    +
    \frac{\eta_1}{2} \left( 2L^2 + \frac{8L^4}{\mu^2}(\frac{\eta_1}{\eta_2} + 1) \right) 4 \eta_2^2 (\tau - 1) (n+1) \frac{\sigmav}{n}\\
    &\quad
    +
    \frac{\eta_1}{2} \left( 2L^2 + \frac{8L^4}{\mu^2}(\frac{\eta_1}{\eta_2} + 1) \right) 16 \eta_1^2 (\tau - 1)^2 \rhow^2\\
    &\quad
    +
    \frac{\eta_1}{2} \left( 2L^2 + \frac{8L^4}{\mu^2}(\frac{\eta_1}{\eta_2} + 1) \right) 16 \eta_2^2 (\tau - 1)^2 \rhov^2\\
    &\quad
    - 
    C_g \frac{1}{T} \sum_{t=0}^{T-1} g_{t}
    - 
    C_G \frac{1}{T} \sum_{t=0}^{T-1} G_{t}.
\end{align}
In above, we use the notation
\begin{align}
    C_g 
    &\coloneqq
    \frac{\eta_1}{2} \left(1 - \eta_1 L_{\Lambda} \right) - \frac{\eta_1^3}{\eta_2} \frac{2 L^2}{\mu^2}  \hat{L}
    -
    \frac{\eta_1}{2} \left( 2L^2 + \frac{8L^4}{\mu^2}(\frac{\eta_1}{\eta_2} + 1) \right) 20 \eta_1^2 (\tau - 1)^2,\\
    C_G
    &\coloneqq
    \frac{\eta_1}{2} \frac{4 L^2}{\mu^2} \left( 1 - \eta_2 L \right)
    -
    \frac{\eta_1}{2} \left( 2L^2 + \frac{8L^4}{\mu^2}(\frac{\eta_1}{\eta_2} + 1) \right) 20 \eta_2^2 (\tau - 1)^2.
\end{align}
Next, we use the assumption that $\frac{\eta_1}{\eta_2} \leq \frac{1}{8 \kappa^2}$ and simplify further as follows
\begin{align}
    \frac{1}{T} \left( \E [ \Lambda(\wbar_{T}) ] - \Lambda(\wbar_0) \right)
    &\leq
    - 
    \frac{\eta_1}{4} \frac{1}{T} \sum_{t=0}^{T-1} \E \norm{\gr \Lambda(\wbar_{t})}^2 \\
    &\quad 
    +
    \frac{\eta_1}{\eta_2 T} \frac{2 L^4}{\mu^3} D^2\\
    &\quad 
    +
     \frac{\eta_1^2}{4}  \hat{L} \frac{\sigmaw}{n} 
    +
    \eta_1^2 \frac{L_{\Lambda}}{2} \frac{\sigmaw}{n}
    +
    \frac{\eta_1}{2} \eta_2 \frac{4 L^3}{\mu^2} \frac{\sigmav}{n}\\
    &\quad 
    +
    \frac{\eta_1}{2} \left( 3L^2 + 8 \kappa^2 L^2 \right) 4 \eta_1^2 (\tau - 1) (n+1) \frac{\sigmaw}{n}\\
    &\quad
    +
    \frac{\eta_1}{2} \left( 3L^2 + 8 \kappa^2 L^2 \right) 4 \eta_2^2 (\tau - 1) (n+1) \frac{\sigmav}{n}\\
    &\quad
    +
    \frac{\eta_1}{2} \left( 3L^2 + 8 \kappa^2 L^2 \right) 16 \eta_1^2 (\tau - 1)^2 \rhow^2\\
    &\quad
    +
    \frac{\eta_1}{2} \left( 3L^2 + 8 \kappa^2 L^2 \right) 16 \eta_2^2 (\tau - 1)^2 \rhov^2.
\end{align}
if $C_g \geq 0$ and $C_G \geq 0$. We further divide both sides by $\frac{\eta_1}{4}$ and conclude that
\begin{align}
    \frac{1}{T} \sum_{t=0}^{T-1} \E \norm{\gr \Lambda(\wbar_{t})}^2
    &\leq
    \frac{4}{\eta_1 T} \left( \Lambda(\wbar_0)  - \E [ \Lambda(\wbar_{T}) ] \right)\\
    &\quad 
    +
    \frac{1}{\eta_2 T} 8 \kappa^3 L D^2
    +
    \eta_1 \hat{L} \frac{\sigmaw}{n} 
    +
    \eta_1 2 L_{\Lambda} \frac{\sigmaw}{n}
    +
    \eta_2 8 \kappa^2 L \frac{\sigmav}{n}\\
    &\quad 
    +
    \eta_1^2 8 L^2 \left( 3 + 8 \kappa^2 \right) (\tau - 1) (n+1) \frac{\sigmaw}{n}\\
    &\quad
    +
    \eta_2^2 8 L^2 \left( 3 + 8 \kappa^2 \right) (\tau - 1) (n+1) \frac{\sigmav}{n}\\
    &\quad
    +
    \eta_1^2 32  L^2 \left( 3 + 8 \kappa^2 \right) (\tau - 1)^2 \rhow^2\\
    &\quad
    +
    \eta_2^2 32 L^2 \left( 3 + 8 \kappa^2 \right) (\tau - 1)^2 \rhov^2.
\end{align}
We further simplify by noting that $\Lambda(\wbar_0)  - \E [ \Lambda(\wbar_{T}) ] \leq \Delta_{\Lambda} \coloneqq \Lambda(\bbw_0) - \min_{\bbw \in \ccalW} \Lambda(\bbw)$ and that $\frac{n+1}{n} \leq 2$,
\begin{align}
    \frac{1}{T} \sum_{t=0}^{T-1} \E \norm{\gr \Lambda(\wbar_{t})}^2
    &\leq
    \frac{4 \Delta_{\Lambda}}{\eta_1 T}
    +
    \frac{1}{\eta_2 T} 8 \kappa^3 L D^2
    +
    \eta_1 (\hat{L} + 2 L_{\Lambda}) \frac{\sigmaw}{n} 
    +
    \eta_2 16 \kappa^2 L \frac{\sigmav}{n}\\
    &\quad 
    +
    \eta_1^2 16 L^2 \left( 3 + 8 \kappa^2 \right) (\tau - 1) \sigmaw
    +
    \eta_2^2 8 L^2 \left( 3 + 8 \kappa^2 \right) (\tau - 1) \sigmav\\
    &\quad
    +
    \eta_1^2 32  L^2 \left( 3 + 8 \kappa^2 \right) (\tau - 1)^2 \rhow^2
    +
    \eta_2^2 32 L^2 \left( 3 + 8 \kappa^2 \right) (\tau - 1)^2 \rhov^2.
\end{align}
Note that for $\eta_2 \leq 1/L$, we have $\hat{L} + 2 L_{\Lambda} \leq 3 (\kappa + 2) L = \ccalO(\kappa L)$. Therefore,
\begin{align}
    \frac{1}{T} \sum_{t=0}^{T-1} \E \norm{\gr \Lambda(\wbar_{t})}^2
    &\leq
    \ccalO \left( \frac{4 \Delta_{\Lambda}}{\eta_1 T} \right)
    +
    \ccalO \left(\frac{\kappa^3 L D^2}{\eta_2 T} \right)
    +
    \ccalO \left(\eta_1 \kappa L\frac{\sigmaw}{n} \right)
    +
    \ccalO \left(\eta_2 \kappa^2 L \frac{\sigmav}{n} \right)\\
    &\quad 
    +
    \ccalO \left( (\eta_1^2  \sigmaw + \eta_2^2  \sigmav) L^2 \kappa^2 \tau  \right)
    +
    \ccalO \left((\eta_1^2  \rhow^2 + \eta_2^2 \rhov^2) L^2 \kappa^2 \tau^2  \right).
\end{align}
The two conditions $C_g \geq 0$ and $C_G \geq 0$ can be simplified as follows:
\begin{align}
    C_g 
    &\coloneqq
    \frac{\eta_1}{2} \left(1 - \eta_1 L_{\Lambda} \right) - \frac{\eta_1^3}{\eta_2} \frac{2 L^2}{\mu^2}  \hat{L}
    -
    \frac{\eta_1}{2} \left( 2L^2 + \frac{8L^4}{\mu^2}(\frac{\eta_1}{\eta_2} + 1) \right) 20 \eta_1^2 (\tau - 1)^2\\
    &\quad
    \geq
    \frac{\eta_1}{2} \left( 1 - \eta_1 (\kappa + 1) L  \right)
    -
    \frac{\eta_1^2}{4} (\kappa + 4) L
    -
    \frac{\eta_1}{2} \left( 3L^2 + 8 \kappa^2 L^2 \right) 20 \eta_1^2 (\tau - 1)^2 \\
    &\quad
    \geq
    \frac{\eta_1}{2} \left( 1 - \frac{9}{2} \eta_1 \kappa L - 220 \eta_1^2 \kappa^2 L^2 (\tau - 1)^2 \right),
\end{align}
where we used the fact that $\hat{L} \coloneqq L + L_{\Lambda} + 2 \eta_2 L^2 \leq (\kappa + 4)L \leq 5 \kappa L$ for $\eta_2 \leq 1/L$. Also, $2L^2 + \frac{8L^4}{\mu^2}(\frac{\eta_1}{\eta_2} + 1) \leq 3L^2 + 8 \kappa^2 L^2 \leq 11 \kappa^2 L^2$ for $\frac{\eta_1}{\eta_2} \leq \frac{1}{8 \kappa^2}$. Moreover for $C_G$, 
\begin{align}
    C_G
    &\coloneqq
    \frac{\eta_1}{2} \frac{4 L^2}{\mu^2} \left( 1 - \eta_2 L \right)
    -
    \frac{\eta_1}{2} \left( 2L^2 + \frac{8L^4}{\mu^2}(\frac{\eta_1}{\eta_2} + 1) \right) 20 \eta_2^2 (\tau - 1)^2 \\
    &\quad
    \geq
    2 \eta_1 \kappa^2 \left( 1 - \eta_2 L - 55 \eta_2^2 L^2 (\tau - 1)^2 \right).
\end{align}
Therefore, if the following conditions hold, then $C_g \geq 0$ and $C_G \geq 0$ hold as well,
\begin{align}
    \frac{9}{2} \eta_1 \kappa L + 220 \eta_1^2 \kappa^2 L^2 (\tau - 1)^2 
    &\leq 
    1\\
    \eta_2 L + 55 \eta_2^2 L^2 (\tau - 1)^2 
    &\leq 
    1\\
    \eta_2 
    &\leq
    \frac{1}{L} \\
    \frac{\eta_1}{\eta_2} 
    &\leq
    \frac{1}{8 \kappa^2},
\end{align}
which together with $32 (\tau - 1)^2 L^2 (\eta_1^2 + \eta_2^2) \leq 1$ are the conditions required in Theorem \ref{thm: opt}.


\subsection{Proof of useful lemmas}

\subsubsection{Proof of Lemma \ref{Lemma: Lambda contraction}} \label{proof Lemma: Lambda contraction}
According to Lemma \ref{lemma:L_Lambda}, gradient of the function $\Lambda(\cdot)$ is $L_{\Lambda}$-Lipschitz. Therefore, we can write
\begin{align} \label{eq: Lambda contraction 1}
    \Lambda(\wbar_{t+1}) - \Lambda(\wbar_{t})
    &\leq
    \left \langle \gr \Lambda(\wbar_{t}) , \wbar_{t+1} - \wbar_{t} \right \rangle
    +
    \frac{L_{\Lambda}}{2} \norm{\wbar_{t+1} - \wbar_{t}}^2 \\
    &=
    -\eta_1 \left \langle \gr \Lambda(\wbar_{t}), \frac{1}{n} \sum_{i \in [n]} \sgrw \Lhat_i(\bbw^i_t , \bbv^i_t) \right \rangle
    +
    \eta_1^2 \frac{L_{\Lambda}}{2} \norm{\frac{1}{n} \sum_{i \in [n]} \sgrw \Lhat_i(\bbw^i_t , \bbv^i_t)}^2 ,
\end{align}
where we use the update rule of \texttt{FedOT-GDA} and note that the difference of averaged models can be written as $\wbar_{t+1} - \wbar_{t} = -\eta_1 \frac{1}{n} \sum_{i \in [n]} \sgrw \Lhat_i(\bbw^i_t , \bbv^i_t)$. Moreover, since the stochastic gradients $\sgrw \Lhat_i$ are unbiased and variance-bounded by $\sigmaw$, we can take expectation from both sides of \eqref{eq: Lambda contraction 1} and further simplify it as follows
\begin{align}
    \E [\Lambda(\wbar_{t+1}) - \E [\Lambda(\wbar_{t})]
    \leq
    - \frac{\eta_1}{2} \E \norm{\gr \Lambda(\wbar_{t})}^2 
    +
    \frac{\eta_1}{2} h_t
    -
    \frac{\eta_1}{2} \left( 1 - \eta_1 L_{\Lambda} \right) g_t
    +
    \eta_1^2 \frac{L_{\Lambda}}{2} \frac{\sigmaw}{n}.
\end{align}


\subsubsection{Proof of Lemma \ref{Lemma: h_t bound}} \label{proof Lemma: h_t bound}
We begin bounding $h_t$ by adding/subtracting the term $\grw \Lhat(\wbar_t, \vbar_t)$ to write
\begin{align} \label{eq: h_t 1}
    h_t
    &=
    \E \norm{\gr \Lambda(\wbar_{t}) - \frac{1}{n} \sum_{i \in [n]} \grw \Lhat_i(\w^i_t , \bbv^i_t)}^2 \\
    &\leq
    2 \E \norm{\gr \Lambda(\wbar_{t}) - \grw \Lhat(\wbar_t, \vbar_t)}^2
    +
    2 \E \norm{\grw \Lhat(\wbar_t, \vbar_t) - \frac{1}{n} \sum_{i \in [n]} \grw \Lhat_i(\w^i_t , \bbv^i_t)}^2.
\end{align}
The first term in RHS of \eqref{eq: h_t 1} can be bounded as follows:
\begin{align}  \label{eq: h_t 2}
    \E \norm{\gr \Lambda(\wbar_t) - \grw \Lhat(\wbar_t, \vbar_t)}^2 
    &=
    \E \norm{\grw \Lhat(\wbar_t, \bbv^*(\wbar_t)) - \grw \Lhat(\wbar_t, \vbar_t)}^2 \\
    &\stackrel{(a)}{\leq}
    L^2 \E \norm{\bbv^*(\wbar_t) - \vbar_t}^2 \\
    &\stackrel{(b)}{\leq}
    \frac{2 L^2}{\mu} \E \left[ \Lambda(\wbar_{t}) - \Lhat(\wbar_{t}, \vbar_{t}) \right] \\
    &\stackrel{(c)}{=}
    \frac{2 L^2}{\mu} b_t.
\end{align}
In inequality $(a)$, we use the assumption that $\Lhat$ has $L$-Lipschitz gradients (Assumption \ref{lemma:L_Lambda}). To derive $(b)$, we again use Assumption \ref{lemma:L_Lambda} in which $\Lhat(\bbw,\cdot)$ is $\mu$-strongly concave. Lastly, $(c)$ is implied from the definition of $b_t$. The second term in RHS of \eqref{eq: h_t 1} can be bounded by noting that the local gradients $\grw \Lhat_i(\cdot, \bbv^i)$ are $L_1$-Lipschitz, which we can write
\begin{align} \label{eq: h_t 3}
    \E \norm{\grw \Lhat(\wbar_t, \vbar_t) - \frac{1}{n} \sum_{i \in [n]} \grw \Lhat_i(\w^i_t , \bbv^i_t)}^2 
    &\!=
    \E \norm{\frac{1}{n} \sum_{i \in [n]} \grw \Lhat_i(\wbar_t , \vbar_t) - \frac{1}{n} \sum_{i \in [n]} \grw \Lhat_i(\w^i_t , \bbv^i_t)}^2 \\
    &\leq
    L^2  \frac{1}{n} \sum_{i \in [n]} \E \norm{\bbw^i_t - \wbar_{t}}^2
    +
    L^2  \frac{1}{n} \sum_{i \in [n]} \E \norm{\bbv^i_t - \vbar_{t}}^2\\
    &=
    L^2 e_t + L^2 E_t.
\end{align}
Finally, plugging \eqref{eq: h_t 2} and \eqref{eq: h_t 3} back in \eqref{eq: h_t 1} implies the claim of the lemma, that is
\begin{align}
    h_t
    \leq
    \frac{4 L^2}{\mu} b_t 
    +
    2 L^2 e_t
    +
    2 L^2 E_t.
\end{align}


\subsubsection{Proof of Lemma \ref{lemma: sum e_t + E_t}} \label{proof Lemma: sum e_t + E_t}

To bound the average $\frac{1}{T} \sum_{t=0}^{T-1} e_{t}+\frac{1}{T} \sum_{t=0}^{T-1} E_{t}$, we first prove the following bounds on $e_t$ and $E_t$ individually.

\begin{prop} \label{prop:e-t and E-t}
If Assumptions \ref{assumption: loss} and \ref{assumption: stoch gr} hold, then
\begin{align} \label{eq: prop e_t}
    e_{t}
    &\leq
    16 \eta_1^2 (\tau - 1) L^2 \sum_{l=t_c+1}^{t-1} e_l 
    +
    16 \eta_1^2 (\tau - 1) L^2 \sum_{l=t_c+1}^{t-1} E_l  \\
    &\quad
    +
    10 \eta_1^2 (\tau - 1) \sum_{l=t_c+1}^{t-1} g_l
    +
    8 \eta_1^2 (\tau - 1)^2 \rhow^2 
    +
    2 \eta_1^2 (\tau - 1) (n+1) \frac{\sigmaw}{n},
\end{align}
and
\begin{align}  \label{eq: prop E_t}
    E_{t}
    &\leq
    16 \eta_2^2 (\tau - 1) L^2 \sum_{l=t_c+1}^{t-1} e_l 
    +
    16 \eta_2^2 (\tau - 1) L^2 \sum_{l=t_c+1}^{t-1} E_l  \\
    &\quad
    +
    10 \eta_2^2 (\tau - 1) \sum_{l=t_c+1}^{t-1} G_l
    +
    8 \eta_2^2 (\tau - 1)^2 \rhov^2 
    +
    2 \eta_2^2 (\tau - 1) (n+1) \frac{\sigmav}{n},
\end{align}
where $t_c$ denotes the index of the most recent server-worker communication, i.e. $t_c = \floor*{\frac{t}{\tau}} \tau$.
\end{prop}

Next, we employ the results in Proposition \ref{prop:e-t and E-t} and write for any $t$ that 
\begin{align} 
    e_t + E_{t}
    &\leq
    16 (\eta_1^2 + \eta_2^2) (\tau - 1) L^2 \sum_{l=t_c+1}^{t-1} (e_l + E_l)\\
    &\quad
    +
    10 \eta_1^2 (\tau - 1) \sum_{l=t_c+1}^{t-1} g_l
    +
    10 \eta_2^2 (\tau - 1) \sum_{l=t_c+1}^{t-1} G_l\\
    &\quad
    +
    8 \eta_1^2 (\tau - 1)^2 \rhow^2 
    +
    8 \eta_2^2 (\tau - 1)^2 \rhov^2 
    +
    2 \eta_1^2 (\tau - 1) (n+1) \frac{\sigmaw}{n}
    +
    2 \eta_2^2 (\tau - 1) (n+1) \frac{\sigmav}{n}.
\end{align}
Now we use Lemma 9 in \cite{reisizadeh2020robust} and conclude that if $32 (\tau - 1)^2 L^2 (\eta_1^2 + \eta_2^2) \leq 1$ holds, then
\begin{align}
    \frac{1}{T} \sum_{t=0}^{T-1} e_{t}
    +
    \frac{1}{T} \sum_{t=0}^{T-1} E_{t}
    &\leq
    20 \eta_1^2 (\tau - 1)^2 \frac{1}{T} \sum_{t=0}^{T-1} g_t
    +
    20 \eta_2^2 (\tau - 1)^2 \frac{1}{T} \sum_{t=0}^{T-1} G_t \\
    &\quad
    +
    16 \eta_1^2 (\tau - 1)^2 \rhow^2
    +
    16 \eta_2^2 (\tau - 1)^2 \rhov^2\\
    &\quad
    +
    4 \eta_1^2 (\tau - 1) (n+1) \frac{\sigmaw}{n}
    +
    4 \eta_2^2 (\tau - 1) (n+1) \frac{\sigmav}{n}.
\end{align}

\subsubsection{Proof of Proposition \ref{prop:e-t and E-t}}

Proof of the two bounds in \eqref{eq: prop e_t} and \eqref{eq: prop E_t} follow the same logic and we provide the proof of \eqref{eq: prop e_t} in the following.  Consider an iteration $t \geq 1$ and let $t_c$ denote the index of the most recent communication between the workers and the server, i.e. $t_c = \floor*{\frac{t}{\tau}} \tau$. All workers share the same local models at iteration $t_c + 1$, i.e. $\w^1_{t_c+1}=\cdots=\w^n_{t_c+1}=\wbar_{t_c+1}$. According to the update rule of \texttt{FedOT-GDA}, we can write for each node $i$ that
\begin{align} \label{eq: prop e_t 1}
    \w^i_{t_c+2} &= \w^i_{t_c+1} - \eta_1 \sgrw \Lhat_i(\w^i_{t_c+1}, \bbv^i_{t_c+1}), \\
    \vdots \\
    \w^i_{t} &= \w^i_{t-1} - \eta_1 \sgrw \Lhat_i(\w^i_{t-1}, \bbv^i_{t-1}).
\end{align}
Summing up all the equalities in \eqref{eq: prop e_t 1} yields that
\begin{align}
    \w^i_{t} &= \w^i_{t_c+1} - \eta_1 \sum_{l=t_c+1}^{t-1} \sgrw \Lhat_i(\w^i_l, \bbv^i_l).
\end{align}
Therefore, the difference of the local models $\w^i_{t}$ and their average $\wbar_{t}$ can be written as
\begin{align}
    \w^i_{t} - \wbar_{t} 
    &= 
    \w^i_{t_c+1} - \eta_1 \sum_{l=t_c+1}^{t-1} \sgrw \Lhat_i(\w^i_l, \bbv^i_l)
    -
    \left(\wbar_{t_c+1} - \eta_1 \frac{1}{n} \sum_{j \in [n]} \sum_{l=t_c+1}^{t-1} \sgrw \Lhat_j(\w^j_l, \bbv^j_l) \right) \\
    &=
    -\eta_1 \left( \sum_{l=t_c+1}^{t-1} \sgrw \Lhat_i(\w^i_l, \bbv^i_l) - \frac{1}{n} \sum_{j \in [n]} \sum_{l=t_c+1}^{t-1} \sgrw \Lhat_j(\w^j_l, \bbv^j_l) \right).
\end{align}
This yields the following bound on each local deviation from the average $\E \Vert \w^i_{t} - \wbar_{t} \Vert^2$:
\begin{align} \label{eq: prop e_t 2}
    \E \norm{\w^i_{t} - \wbar_{t}}^2 
    &=
    \eta_1^2 \E \norm{\sum_{l=t_c+1}^{t-1} \sgrw \Lhat_i(\w^i_l, \bbv^i_l) - \frac{1}{n} \sum_{j \in [n]} \sum_{l=t_c+1}^{t-1} \sgrw \Lhat_j(\w^j_l, \bbv^j_l)}^2 \\
    &\leq
    2 \eta_1^2 
    \E \norm{\sum_{l=t_c+1}^{t-1} \sgrw \Lhat_i(\w^i_l, \bbv^i_l) }^2
    +
    2 \eta_1^2 
    \E \norm{\frac{1}{n} \sum_{j \in [n]} \sum_{l=t_c+1}^{t-1} \sgrw \Lhat_j(\w^j_l, \bbv^j_l)}^2\\
    &{\leq}
    2 \eta_1^2 
    \underbrace{\E \norm{\sum_{l=t_c+1}^{t-1} \grw \Lhat_i(\w^i_l, \bbv^i_l) }^2 }_{T_1}
    +
    2 \eta_1^2 
    \underbrace{\E \norm{\frac{1}{n} \sum_{j \in [n]} \sum_{l=t_c+1}^{t-1} \grw \Lhat_j(\w^j_l, \bbv^j_l)}^2 }_{T_2} \\
    &\quad +
    2 \eta_1^2 (t - t_c - 1) (n+1) \frac{\sigmaw}{n},
\end{align}
where we used Assumption \ref{assumption: stoch gr}. The term $T_2$ in \eqref{eq: prop e_t 2} can simply be bounded as
\begin{align}
    T_2
    \leq
    \E \norm{\frac{1}{n} \sum_{j \in [n]} \sum_{l=t_c+1}^{t-1} \grw \Lhat_j(\w^j_l, \bbv^j_l)}^2
    \leq
    (t - t_c - 1) \sum_{l=t_c+1}^{t-1} \E \norm{\frac{1}{n} \sum_{j \in [n]} \grw \Lhat_j(\w^j_l, \bbv^j_l)}^2
\end{align}
Note that $t_c$ denotes the latest server-worker communication before iteration $t$, hence $t - t_c \leq \tau$ where $\tau$ is the duration of local updates in each round. Therefore, we have
\begin{align} \label{eq: prop e_t 4}
    T_2
    \leq
    (\tau - 1) \sum_{l=t_c+1}^{t-1} \E \norm{\frac{1}{n} \sum_{j \in [n]} \grw \Lhat_j(\w^j_l, \bbv^j_l)}^2
    \leq
    (\tau - 1) \sum_{l=t_c+1}^{t-1} g_l
\end{align}
Now we proceed to bound the term $T_1$ in \eqref{eq: prop e_t 2} as follows:
\begin{align}
    T_1
    &=
    \E \norm{\sum_{l=t_c+1}^{t-1} \grw \Lhat_i(\w^i_l, \bbv^i_l) }^2 \\
    &\leq
    (\tau - 1) \sum_{l=t_c+1}^{t-1} \E \norm{\grw \Lhat_i(\w^i_l, \bbv^i_l) }^2 \\
    &\leq
    4 (\tau - 1) \sum_{l=t_c+1}^{t-1} \E \norm{\grw \Lhat_i(\w^i_l, \bbv^i_l) - \grw \Lhat_i(\wbar_l, \vbar_l) }^2 \\
    &\quad +
    4 (\tau - 1) \sum_{l=t_c+1}^{t-1} \E \norm{\grw \Lhat_i(\wbar_l, \vbar_l) - \frac{1}{n} \sum_{j \in [n]} \grw \Lhat_j(\wbar_l, \bbv^j_l) }^2 \\
    &\quad +
    4 (\tau - 1) \sum_{l=t_c+1}^{t-1} \E \norm{\frac{1}{n} \sum_{j \in [n]} \grw \Lhat_j(\wbar_l, \vbar_l) - \frac{1}{n} \sum_{j \in [n]} \grw \Lhat_j(\w^j_l, \bbv^j_l) }^2 \\
    &\quad +
    4 (\tau - 1) \sum_{l=t_c+1}^{t-1} \E \norm{\frac{1}{n} \sum_{j \in [n]} \grw \Lhat_j(\w^j_l, \bbv^j_l) }^2
\end{align}
We can simply this bound by using Assumption \ref{assumption: loss} on Lipschitz gradients for the local objectives $\Lhat_i$s and applying the notations for $e_l$ and $g_l$ to derive
\begin{align} \label{eq: prop e_t 3}
    T_1
    &\leq
    4 (\tau - 1) L^2 \sum_{l=t_c+1}^{t-1} \left( \E \norm{\w^i_{l} - \wbar_{l}}^2  + \E \norm{\bbv^i_{l} - \vbar_{l}}^2 \right)\\
    &\quad 
    +
    4 (\tau - 1) \sum_{l=t_c+1}^{t-1} \E \norm{\grw \Lhat_i(\wbar_l, \vbar_l) - \grw \Lhat(\wbar_l, \vbar_l) }^2 \\
    &\quad +
    4 (\tau - 1) L^2 \sum_{l=t_c+1}^{t-1} (e_l + E_l)
    +
    4 (\tau - 1) \sum_{l=t_c+1}^{t-1} g_l
\end{align}
We can plug \eqref{eq: prop e_t 4} and \eqref{eq: prop e_t 3} into \eqref{eq: prop e_t 2} and take the average of the both sides over $i = 1,\cdots,n$. This, together with Assumption \ref{assumption: loss} (ii) and (iii) implies that
\begin{align}
    e_{t}
    &\leq
    16 \eta_1^2 (\tau - 1) L^2 \sum_{l=t_c+1}^{t-1} e_l 
    +
    16 \eta_1^2 (\tau - 1) L^2 \sum_{l=t_c+1}^{t-1} E_l  \\
    &\quad
    +
    10 \eta_1^2 (\tau - 1) \sum_{l=t_c+1}^{t-1} g_l
    +
    8 \eta_1^2 (\tau - 1)^2 \rhow^2 
    +
    2 \eta_1^2 (\tau - 1) (n+1) \frac{\sigmaw}{n}.
\end{align}


\subsubsection{Proof of Lemma \ref{lemma: sum b_t}} \label{proof lemma: sum b_t}

We first establish the following bound on $b_t$.

\begin{prop} \label{prop: b_t contraction}
If Assumptions \ref{assumption: loss} and \ref{assumption: stoch gr} hold, then the sequence of $\{b_t\}_{t \geq 0}$ iterations satisfies the following contraction bound:
\begin{align}  \label{eq:b_t-final}
    b_{t+1} 
    &\leq
    (1 - \mu \eta_2) \left( 1 + \eta_1 \frac{4 L^2}{\mu} \right) b_t
    +
    \frac{\eta_1}{2} \E \norm{\gr \Lambda(\wbar_{t})}^2 \\
    &\quad
    +
    \frac{\eta_1^2}{2} \left( L + L_{\Lambda} + 2 \eta_2 L^2 \right) g_t
    -
    \frac{\eta_2}{2} \left( 1 - \eta_2 L\right) G_t\\
    &\quad +
    L^2 \left( \eta_1  + \eta_2 \right) e_t
    +
    L^2 \left( \eta_1  + \frac{\eta_2}{2} \right) E_t\\
    &\quad
    +
    \frac{\eta_1^2}{2} \left( L + L_{\Lambda} + 2 \eta_2 L^2 \right) \frac{\sigmaw}{n} 
    +
    \frac{\eta_2^2}{2} L \frac{\sigmav}{n},
\end{align}
where $L_{\Lambda}$ is the Lipschitz gradient parameter of the function $\Lambda(\cdot)$ characterized in Lemma \ref{lemma:L_Lambda}.
\end{prop}

Having set the above contraction bound on $b_t$, we can bound the average over iterations as follows. Consider the coefficient of $b_t$ in \eqref{eq:b_t-final}. If the stepsizes satisfy $\frac{\eta_1}{\eta_2} \leq \frac{1}{8 \kappa^2}$, then we have
\begin{align}  
    b_{t+1} 
    &\leq
    \left(1 - \frac{\mu}{2} \eta_2 \right)  b_t
    +
    \frac{\eta_1}{2} \E \norm{\gr \Lambda(\wbar_{t})}^2 \\
    &\quad
    +
    \frac{\eta_1^2}{2} \left( L + L_{\Lambda} + 2 \eta_2 L^2 \right) g_t
    -
    \frac{\eta_2}{2} \left( 1 - \eta_2 L\right) G_t\\
    &\quad +
    L^2 \left( \eta_1  + \eta_2 \right) e_t
    +
    L^2 \left( \eta_1  + \frac{\eta_2}{2} \right) E_t\\
    &\quad
    +
    \frac{\eta_1^2}{2} \left( L + L_{\Lambda} + 2 \eta_2 L^2 \right) \frac{\sigmaw}{n} 
    +
    \frac{\eta_2^2}{2} L \frac{\sigmav}{n},
\end{align}
We can write the above contraction for all $t=0,\cdots,T-1$ which yields that
\begin{align} 
    \frac{1}{T} \sum_{t=0}^{T-1} b_{t}
    &\leq
    \frac{L^2}{\mu^2} \frac{D^2}{\eta_2 T}
    +
    \frac{\eta_1}{\eta_2} \frac{1}{\mu} \frac{1}{T} \sum_{t=0}^{T-1} \E \norm{\gr \Lambda(\wbar_{t})}^2 \\
    &\quad +
    \frac{\eta_1^2}{\eta_2} \frac{1}{\mu_2 n} \left( L + L_{\Lambda} + 2 \eta_2 L^2 \right) \frac{1}{T} \sum_{t=0}^{T-1} g_{t}
    -
    \frac{1}{\mu} (1 - \eta_2 L) \frac{1}{T} \sum_{t=0}^{T-1} G_{t}\\
    &\quad +
    +
    \frac{\eta_1 + \eta_2}{\eta_2} \frac{2L^2}{\mu} \left( \frac{1}{T} \sum_{t=0}^{T-1} e_{t} + \frac{1}{T} \sum_{t=0}^{T-1} E_{t}\right) \\
    &\quad +
    \frac{\eta_1^2}{\eta_2} \frac{1}{\mu} \left(  L + L_{\Lambda} + 2 \eta_2 L^2 \right) \frac{\sigmaw}{n} 
    +
    \eta_2 \frac{L}{\mu} \frac{\sigmav}{n},
\end{align}
concluding the proof of Lemma \ref{lemma: sum b_t}.

\subsubsection{Proof of Proposition \ref{prop: b_t contraction}} \label{proof prop: b_t contraction}

We first note that according to Assumption \ref{assumption: loss}, gradients $\grv \Lhat(\w, \cdot)$ are $L$-Lipschitz. We can therefore write
\begin{align} \label{eq:b_t bound 1}
    \Lambda(\wbar_{t+1}) - \Lhat(\wbar_{t+1}, \vbar_{t+1}) 
    \leq
    \Lambda(\wbar_{t+1}) - \Lhat(\wbar_{t+1}, \vbar_{t})
    -
    \langle \gr_{\vbar} \Lhat(\wbar_{t+1}, \vbar_{t}) , \vbar_{t+1} - \vbar_{t} \rangle 
    +
    \frac{L}{2} \norm{\vbar_{t+1} - \vbar_{t}}^2. 
\end{align}
Next, we use the fact that $\vbar_{t+1} - \vbar_t = \eta_2 \frac{1}{n} \sum_{i \in [n]} \sgrv \Lhat_i(\bbw^i_t, \bbv^i_t)$ and take expectations from both sides of \eqref{eq:b_t bound 1}, which yields
\begin{align} \label{eq:b_t bound 2}
    \Lambda(\wbar_{t+1}) - \E \Lhat(\wbar_{t+1}, \vbar_{t+1})
    &\leq
    \Lambda(\wbar_{t+1}) - \Lhat(\wbar_{t+1}, \vbar_{t})
    -
    \frac{\eta_2 }{2} \norm{\grv f(\wbar_{t+1}, \vbar_{t})}^2\\
    &\quad
    +
    \frac{\eta_2 }{2} \norm{\grv f(\wbar_{t+1}, \vbar_{t}) - \frac{1}{n} \sum_{i \in [n]} \sgrv \Lhat_i(\bbw^i_t, \bbv^i_t)}^2 \\
    &\quad
    -
    \frac{\eta_2 }{2} (1 - \eta_2 L) G_t
    +
    \eta_2^2 \frac{L}{2} \frac{\sigmav}{n}.
\end{align}

Now, we recall from Assumption \ref{assumption: loss} (ii) that $\Lhat(\wbar_{t+1}, \cdot)$ is $\mu$-strongly concave, implying that $\Vert \grv f(\wbar_{t+1}, \vbar_{t}) \Vert^2 \geq 2 \mu (\Lambda(\wbar_{t+1}) - \Lhat(\wbar_{t+1}, \vbar_{t}))$.  Therefore, we have that
\begin{align} \label{eq:b_t bound 3}
    \Lambda(\wbar_{t+1}) - \E \Lhat(\wbar_{t+1}, \vbar_{t+1})
    &\leq
    (1 - \mu \eta_2) \left( \Lambda(\wbar_{t+1}) - \Lhat(\wbar_{t+1}, \vbar_{t}) \right)\\
    &\quad
    +
    \frac{\eta_2 }{2} \norm{\grv f(\wbar_{t+1}, \vbar_{t}) - \frac{1}{n} \sum_{i \in [n]} \sgrv \Lhat_i(\bbw^i_t, \bbv^i_t)}^2 \\
    &\quad
    -
    \frac{\eta_2 }{2} (1 - \eta_2 L) G_t
    +
    \eta_2^2 \frac{L}{2} \frac{\sigmav}{n}.
\end{align}
Next, we continue to bound the second term in RHS of \eqref{eq:b_t bound 3}. According to Assumption \ref{assumption: loss} (ii), we can write
\begin{align} \label{eq:b_t bound 4}
    \norm{\grv f(\wbar_{t+1}, \vbar_{t}) - \frac{1}{n} \sum_{i \in [n]} \sgrv \Lhat_i(\bbw^i_t, \bbv^i_t)}^2
    &\leq
    L^2 \frac{1}{n} \sum_{i \in [n]} \norm{\wbar_{t+1} - \w^i_t}^2 
    +
    L^2 \frac{1}{n} \sum_{i \in [n]} \norm{\vbar_{t} - \bbv^i_t}^2 \\
    &\leq
    2L^2 e_t + L^2 E_t + 2L^2 \norm{\wbar_{t+1} - \wbar_{t}}^2.
\end{align}
We can bound the last term above $\norm{\wbar_{t+1} - \wbar_{t}}^2$ as follows
\begin{align} \label{eq:b_t bound 5}
    \E \norm{\wbar_{t+1} - \wbar_{t}}^2
    &=
    \eta_1^2 \E \norm{\frac{1}{n} \sum_{i \in [n]} \sgrw \Lhat_i(\w^i_t , \bbv^i_t)}^2 \\
    &\leq
    \eta_1^2 \E \norm{\frac{1}{n} \sum_{i \in [n]} \grw \Lhat_i(\w^i_t , \bbv^i_t)}^2    
    +
    \eta_1^2 \frac{\sigmaw}{n} \\
    &=
    \eta_1^2 g_t + \eta_1^2 \frac{\sigmaw}{n},
\end{align}
which together with \eqref{eq:b_t bound 4} yields that 
\begin{align} \label{eq:b_t bound 6}
    \E \norm{\grv f(\wbar_{t+1}, \vbar_{t}) - \frac{1}{n} \sum_{i \in [n]} \sgrv \Lhat_i(\bbw^i_t, \bbv^i_t)}^2
    &\leq
    2L^2 e_t + L^2 E_t + 2 \eta_1^2  L^2 g_t + 2 \eta_1^2  L^2 \frac{\sigmaw}{n}.
\end{align}
Thus far, we have bounded $b_{t+1} = \E[\Lambda(\wbar_{t+1}) - \Lhat(\wbar_{t+1}, \vbar_{t+1})]$ as follows
\begin{align} \label{eq: b_t bound 8}
    b_{t+1}
    \leq
    (1 - \mu \eta_2) \E \left[ \Lambda(\wbar_{t+1}) - \Lhat(\wbar_{t+1}, \vbar_{t}) \right] 
    +
    \eta_2 L^2 e_t 
    +
    \frac{\eta_2}{2} L^2 E_t 
    +
    \eta_1^2 \eta_2 L^2 g_t 
     +
    \eta_1^2 \eta_2 L^2 \frac{\sigmaw}{n}
    +
    \eta_2^2 \frac{L}{2}  \frac{\sigmav}{n}
\end{align}

To bound the term $\E [ \Lambda(\wbar_{t+1}) - \Lhat(\wbar_{t+1}, \vbar_{t}) ]$, we can decompose it to the following three terms:
\begin{align} \label{eq:decompose1}
    \Lambda(\wbar_{t+1}) - \Lhat(\wbar_{t+1}, \vbar_{t})
    &=
    \Lambda(\wbar_{t}) - \Lhat(\wbar_{t}, \vbar_{t}) + \Lhat(\wbar_{t}, \vbar_{t}) - \Lhat(\wbar_{t+1}, \vbar_{t}) + \Lambda(\wbar_{t+1}) - \Lambda(\wbar_{t}).
\end{align}
Using Lipschitz gradients in Assumption \ref{assumption: loss}, we can write
\begin{align} \label{eq: b_t bound 7}
    \Lhat(\wbar_{t}, \vbar_{t}) - \Lhat(\wbar_{t+1}, \vbar_{t})
    \leq
    - \langle \grw \Lhat(\wbar_t, \vbar_t), \wbar_{t+1} - \wbar_{t} \rangle 
    +
    \frac{L}{2} \norm{\wbar_{t+1} - \wbar_{t}}^2,
\end{align}
where $\overline{w}_{t+1} - \overline{w}_{t} = - \eta_1 \frac{1}{n} \sum_{i \in [n]} \sgrw f^i(\w^i_t , \bpsi^i_t)$. Taking expectation from both sides of \eqref{eq: b_t bound 7} implies that
\begin{align} \label{eq:decompose2}
    \E \left[ \Lhat(\wbar_{t}, \vbar_{t}) - \Lhat(\wbar_{t+1}, \vbar_{t}) \right]
    &\leq
    \eta_1 \E \norm{\grw \Lhat(\wbar_t, \vbar_t) - \gr \Lambda(\wbar_t)}^2
    +
    \eta_1 \E \norm{\gr \Lambda(\wbar_t)}^2 \\
    &\quad +
    \left( \frac{\eta_1}{2} + \eta_1^2 \frac{L_1}{2} \right) g_t
    +
    \eta_1^2 \frac{L}{2} \frac{\sigmaw}{n} \\
    &\leq
    \eta_1 \frac{2 L^2}{\mu} b_t
    \eta_1 \E \norm{\gr \Lambda(\wbar_t)}^2 
    +
    \left( \frac{\eta_1}{2} + \eta_1^2 \frac{L_1}{2} \right) g_t
    +
    \eta_1^2 \frac{L}{2} \frac{\sigmaw}{n}.
\end{align}
In above, we used the notation $ b_t = \E [ \Lambda(\wbar_t) - \Lhat(\wbar_t, \vbar_t)]$  to write
\begin{align} 
    \E \norm{\gr \Lambda(\wbar_t) - \grw \Lhat(\wbar_t, \vbar_t)}^2 
    &=
    \E \norm{\grw \Lhat(\wbar_t, \vbar^*(\wbar_t)) - \grw \Lhat(\wbar_t, \vbar_t)}^2 \\
    &\leq
    L^2 \E \norm{\vbar^*(\wbar_t) - \vbar_t}^2_F \\
    &\leq
    \frac{2 L^2}{\mu} \E \left[ \Lambda(\wbar_{t}) - \Lhat(\wbar_{t}, \vbar_{t}) \right] \\
    &=
    \frac{2 L^2}{\mu} b_t.
\end{align}
Together with the bound on $\E [ \Lambda(\wbar_{t+1}) ] - \E [ \Lambda(\wbar_{t}) ]$ derived in Lemma \ref{Lemma: Lambda contraction} we conclude the claim in Proposition \ref{prop: b_t contraction}.


\subsubsection{Proof of Lemma \ref{lemma: sum h_t}} \label{proof lemma: sum h_t}

We use the result of Lemma \ref{proof Lemma: h_t bound} and take the average over iterations $t=0,\cdots,T-1$, implying that
\begin{align} 
    \frac{1}{T} \sum_{t=0}^{T-1} h_t
    \leq
    \frac{4 L^2}{\mu} \frac{1}{T} \sum_{t=0}^{T-1} b_t 
    +
    2 L^2 \frac{1}{T} \sum_{t=0}^{T-1} (e_t + E_t).
\end{align}
Now, we employ the bounds on $\frac{1}{T} \sum_{t=0}^{T-1} b_t $ and $\frac{1}{T} \sum_{t=0}^{T-1} (e_t + E_t)$ derived in Lemmas \ref{lemma: sum b_t} and \ref{lemma: sum e_t + E_t}, respectively, which concludes the proof.

\end{document}